%% file: main.tex
\documentclass[11pt]{article}
\usepackage{mathrsfs,amsmath,amsfonts,amssymb,bm,bbm,dsfont,pifont,amscd,stmaryrd,euscript,amsthm,color,epsfig,xr,yfonts,tikz,verbatim}

\usepackage{tikz-cd} 
\usepackage{authblk}
\usepackage[utf8]{inputenc} 
\usepackage[T1]{fontenc}    
\usepackage{url}            
\usepackage{booktabs}       
\usepackage{amsfonts}  
\usepackage{amsmath} 
\usepackage{nicefrac}       
\usepackage{microtype}      
\usepackage{xcolor}
\usepackage{hyperref}  
\usepackage{graphicx}
\usepackage{array}
\usepackage{subcaption,mathrsfs}
\usepackage{natbib}
\usepackage{wrapfig}
\usepackage{tikz}
\usepackage{include/titletoc}
\usepackage[makeroom]{cancel}
\usepackage{algorithm}
\usepackage{algpseudocode}

\newtheorem{thm}{Theorem}
\newtheorem{lem}{Lemma}

\newtheorem{prop}{Proposition}

\newtheorem{assum}{Assumption}
\newtheorem{rem}{Remark}

\newtheorem{defn}{Definition}

\DeclareMathOperator{\E}{\mathbb{E}}

\DeclareMathOperator{\Tr}{Tr}

\usepackage{enumitem}

\usepackage{enumitem}

\title{Efficient Inference after Directionally Stable Adaptive Experiments}

\author[1]{Zikai Shen$^{*}$}
\author[2]{Houssam Zenati$^{*}$}
\author[3,4]{Nathan Kallus}
\author[2,5]{Arthur Gretton}
\author[6]{Koulik Khamaru}
\author[4]{Aurélien Bibaut}

\affil[1]{University College London}
\affil[2]{Gatsby Computational Neuroscience Unit,
University College London}
\affil[3]{Cornell University}
\affil[4]{Netflix}
\affil[5]{Google DeepMind}
\affil[6]{Rutgers University}

\begin{document}

\maketitle
\setcounter{footnote}{0}
\renewcommand{\thefootnote}{\fnsymbol{footnote}} 
\footnotetext[1]{Equal contribution. }

\begin{abstract}
We study inference on scalar-valued pathwise differentiable targets after adaptive data collection, such as a bandit algorithm.
We introduce a novel target-specific condition, \textit{directional stability},
which is strictly weaker than previously imposed target-agnostic stability conditions. 
Under directional stability, we show that estimators that would have been efficient under i.i.d.\ data remain asymptotically normal and semiparametrically efficient when computed from adaptively collected trajectories. The canonical gradient has a martingale form, and directional stability guarantees stabilization of its predictable quadratic variation, enabling high-dimensional asymptotic normality. We characterize efficiency using a convolution theorem for the adaptive-data setting, and give a condition under which the one-step estimator attains the efficiency bound.
We verify directional stability for LinUCB, yielding the first semiparametric efficiency guarantee for a regular scalar target under LinUCB sampling.

\end{abstract}

\section{Introduction}
Statistical inference under adaptive data collection is now routine in modern learning systems. In contextual bandits and related online decision problems, actions are chosen using past observations, inducing dependence between observations collected at different rounds \citep{auer2002finite, li2010contextual, lattimore2020, sutton2018reinforcement}. This dependence can invalidate classical i.i.d.\ asymptotics: even when estimators remain consistent, their limiting distributions may be non-normal, complicating inference \citep{vanderVaart1998asymptotic, hall1980, hadad2021confidence, bibaut2025demystifying}.

A key mechanism by which classical inference can be recovered is \emph{stability} of the adaptive design. Specifically, in the linear regression setting, if the realized random design matrix becomes deterministic as the horizon gets large, then the predictable quadratic variation convergence condition of martingale central limit theorems is satisfied in the analysis of the OLS estimator, which guarantees asymptotic normality \citep{laiweiaos1982}. 
Most existing stability formulations are \emph{full-matrix} conditions: they require the entire empirical covariance (or information) matrix to stabilize after deterministic rescaling. While powerful, such global requirements can be misaligned with regret-minimizing objectives that limit exploration. In particular, modern bandit algorithms deliberately concentrate sampling in ``good'' directions, so information typically accumulates anisotropically across directions. Full-matrix stability can therefore be too regret-expensive to enforce. Fortunately, even it fails, inference for a specific \emph{scalar} target of interest may still be possible.

We introduce a novel notion of stability which is sufficient for inference for \emph{scalar pathwise differentiable targets}. Our starting point is that for scalar targets, inference depends on the design only in the direction in which the target of interest depends on the coefficient vector. We introduce a target-specific notion we call \emph{directional stability}, which requires stabilization of the empirical design only along a relevant direction (at an explicit rate), while allowing instability in directions that do not affect the target. Directional stability can be substantially weaker than classical full-matrix stability \citep{laiweiaos1982} and is compatible with the anisotropic exploration patterns induced by regret minimization, as we show by example with LinUCB.  Finally, anisotropic stabilization is not unique to optimistic exploration. In the multi-armed setting,
recent work by~\cite{han2026thompson} shows that Thompson sampling exhibits a sharp stability dichotomy:
the pull counts are asymptotically deterministic for each suboptimal arm (and for the unique optimal arm),
whereas when there are multiple optimal arms, the vector of optimal-arm pull proportions converges to a
non-degenerate random limit characterized as the invariant law of an SDE. This provides a complementary
example in which stabilization holds in suboptimal-arm directions but fails along directions associated with
the optimal arms. 

\paragraph{Main message: in directionally stable designs, i.i.d.-efficient estimators remain efficient.}
Our main consequence is simple: under directional stability, the estimators that would have been asymptotically normal and efficient under i.i.d.\ sampling are also the ones that are asymptotically normal and efficient under adaptive sampling. No alterations or square-root-propensity weights are needed. We work at the trajectory level, viewing the observed bandit history as a single draw from a horizon-indexed longitudinal experiment. We derive the trajectory-level canonical gradient for the target and show it has a martingale form. Directional stability then guarantees stabilization of the martingale's (conditional) quadratic variation, yielding asymptotic normality in high-dimensional regimes. We further develop an efficiency theory for a sequence of horizon-indexed experiments and show that the resulting one-step estimator achieves the semiparametric efficiency bound associated with the stable limit experiment. In particular, once directional stability holds, there is no intrinsic need for propensity-weighting or ad hoc variance stabilization: the classical one-step construction from the i.i.d.\ setting is already optimal, while propensity weighting would actually degrade efficiency.

\paragraph{Relation to prior work.}
Several lines of work obtain valid inference under adaptivity by modifying estimating
equations to enforce normality under broad sampling rules. Prominent examples include
propensity-weighted or variance-stabilized doubly robust scores, which often require either
known assignment probabilities or consistent propensity estimation and conditional variance
corrections \citep[e.g.,][]{hadad2021confidence, bibaut2021post, zhang2021statistical}.
A complementary literature develops always-valid procedures (e.g., confidence sequences)
that remain valid under arbitrary stopping, sometimes at the cost of conservatism at fixed
horizons \citep{howard2021time_uniform, waudby_smith2021tu_clt, bibaut2022near}.
Another strand characterizes non-Gaussian limits induced by adaptive designs and performs
inference by inverting the corresponding limit experiments \citep{rosenberger1999bootstrap,
hirano2023asymptotic, adusumilli2023optimal, cho2025simulationadaptive}. A review of these different lines is given in \citep{bibaut2025demystifying}.

A distinct and increasingly active line of work studies when \emph{classical Wald-type inference}
is restored under adaptive sampling via \emph{stability} of the empirical design, in the sense of
\citet{laiweiaos1982}. This perspective has motivated stability-based analyses for a
range of bandit algorithms, including UCB-type methods and refinements thereof
\citep{kalvit2021closer, khamaru2024inference, fan2022typical, han2024ucbalgorithmsmultiarmedbandits},
more precise and quantitative characterizations of (in)stability and coverage behavior
\citep{fan2024precise, han2026thompson}, stable variants of Thompson sampling under variance inflation
\citep{halder2025stablethompsonsamplingvalid}, and linear contextual bandits such as LinUCB
\citep{fan2025statisticalinferenceadaptivesampling}. At the same time, these works highlight that
many widely used adaptive policies fail to satisfy full-matrix stability, leading to systematic
under-coverage of naive Wald intervals when applied without adjustment
\citep{fan2024precise}.

Our contribution is a weakening of the notion of stability: given a target, we identify a minimal statistical functional of the design of which the convergence along a deterministic sequence suffices for asymptotic normality and a characterization of the efficiency bound.

\paragraph{Application to LinUCB.}
Finally, we instantiate our conditions for LinUCB. Full-matrix stability is not currently known to hold in linear contextual bandits due to direction-dependent information growth. Directional stability, by construction, matches this anisotropy: it only asks for stabilization in the target direction. Plugging in the characterization of the eigendecomposition of the empirical design matrix under LinUCB from \cite{fan2025statisticalinferenceadaptivesampling}, we verify directional stability and obtain---to our knowledge---the first semiparametric efficiency guarantee for inference on a regular scalar target under LinUCB sampling.

Our contributions are as follows:
\begin{itemize}
    \item \textbf{Trajectory canonical gradient and one-step estimator.} We derive the canonical gradient for a functional of the repeated factor (that is the environment factor, as opposed to the design factors) in the distribution of the full bandit trajectory and construct a one-step estimator that is algebraically identical to the i.i.d.\ one-step estimator (Sections~\ref{sec:canonical_grad}--\ref{sec:one_step_estimator}).
    \item \textbf{Directional stability and high-dimensional asymptotics.} We introduce directional stability and show it yields asymptotic normality via stabilization of predictable quadratic variation, including high-dimensional regimes where plug-in OLS asymptotics are inadequate. We then verify directional stability for LinUCB (Section~\ref{sec: asymptotic_analysis}).
    \item \textbf{Efficiency theory for horizon-indexed experiments.} 
    We use the notion of regularity along a sequence of submodels from \cite{vanderlaan2026nonparametricinstrumentalvariableinference} and efficiency w.r.t. such regular estimators. We characterize the efficiency bound under directional stability and show that the one-step estimator attains it (Section~\ref{sec: efficiency_theory}). 
\end{itemize}

\subsection{Data and Statistical Models}

\paragraph{Data} We observe an adaptively collected sequence $\{O_t\}_{t=1}^T$ with $O_t := (X_t,A_t,Y_t)$, where $X_t\in\mathcal X$ is a (potentially continuous) context, $A_t\in\mathcal A=\{1,\dots,K\}$ is a categorical action, and $Y_t\in\mathcal Y\subset\mathbb R$ is an outcome. Let $\bar O_t := (O_1,\dots,O_t)$ denote the history and $\mathcal F_t := \sigma(\bar O_t)$ the associated filtration. The data are generated by a contextual bandit--type experiment: at each round $t$, the agent selects a (random) logging policy $g_t(\cdot\mid x, \bar{O}_{t-1})$ based on past observations, then observes a fresh, i.i.d. sampled context $X_t$, draws an action $A_t\sim g_t(\cdot\mid X_t, \bar{O}_{t-1})$, and finally observes an outcome $Y_t$ drawn from an unknown model conditional on $(X_t,A_t)$.

 We assume that $X_t$ is independent of $\mathcal F_{t-1}$ with stationary marginal distribution $Q_{0,X}$. Conditional on $(\mathcal F_{t-1},X_t)$, the action satisfies $\mathbb P(A_t=a\mid \mathcal F_{t-1},X_t)=g_t(a\mid X_t, \bar{O}_{t-1})$ for all $a\in\mathcal A$, where $g_t$ is $\mathcal F_{t-1}$-measurable and maps $\mathcal X$ to the $K$-simplex. Conditional on $(X_t,A_t)$, the outcome is independent of the past and has stationary conditional distribution $Y_t \mid (X_t=x,A_t=a)\sim Q_{0,Y}(\cdot\mid a,x)$. 

\paragraph{Statistical models.} We suppose that the full trajectory $\bar{O}_T$ is a draw of $P^{(T)}$ living in the set $\mathcal{M}_T^{\mathrm{np}}$ of distributions that are absolutely continuous w.r.t. an appropriate product measure $\mu^{(T)} = \mu^{\otimes T}$, where density w.r.t $\mu^{(T)}$ factors as 
\begin{align}
    \frac{dP^{(T)}}{d \mu^{(T)}}(\bar{o}_T) = \prod_{t=1}^{T}q_X(x_t)g_t(a_t\mid x_t, \bar{o}_{t-1})q_{Y}(y_t\mid a_t, x_t),
    \label{eq:model_factorization}
\end{align}
where the likelihood factors $q_X, q_Y, g_t, 1\leq t\leq T$ are \emph{unknown} to the statistician and allowed to vary arbitrarily. We assume that $\mu(\mathcal{X}\times \mathcal{A}) = 1$. We then specialize to a restricted setting in the next subsection.  For each horizon $T$, the data-generating process induces a statistical experiment
$\mathcal M_T^{\mathrm{np}}$ consisting of all distributions $P^{(T)}$ on $\bar O_T$
admitting the factorization \eqref{eq:model_factorization}.
We emphasize that $(\mathcal M_T^{\mathrm{np}})_{T\ge 1}$ is a \emph{sequence of statistical models}. 

\subsection{Sequence of Target Estimands}

We now define a sequence of target estimands indexed by the experimental horizon $T$.
Each target $\Psi_T$ is defined on the corresponding statistical model
$\mathcal M_T$ and depends on a feature representation whose dimension may grow with $T$.
All asymptotic statements in the sequel are understood along this sequence as $T \to \infty$. Throughout, let $\mathcal{Z} = \mathcal{X}\times \mathcal{A}$. We write $z = (x,a)$, and $Z_t = (X_t, A_t)$. For each horizon $T$, define a feature map
$\varphi_T : \mathcal{Z} \to \mathbb R^{d_T}$.

\paragraph{Notation.}
For any measurable $f:\mathcal Z\times\mathcal Y\to\mathbb R$, define the time-averaged empirical process notation
\[
\bar P^{(T)}f:=\frac1T\sum_{t=1}^T\E_{P^{(T)}}[f(Z_t,Y_t)],
\]
which integrates over the arguments of $f$ only, even when $f$ is random. This induces the Hilbert space
\begin{align}
\label{eq: l2_bar_pt_defn}
    L^2(\bar P^{(T)}):=\{f:\mathcal Z\times\mathcal Y\to\mathbb R\mid \bar P^{(T)}(f^2)<\infty\},
\qquad
\|f\|_{L^2(\bar P^{(T)})}:=\{\bar P^{(T)}(f^2)\}^{1/2}.
\end{align}
Similarly, for any measurable $f:\mathcal Z\times\mathcal Y\to\mathbb R$, define the empirical average
\[
\bar P_T f:=\frac1T\sum_{t=1}^T f(X_t,A_t,Y_t),
\]
with associated (random) seminorm $\|f\|_{L^2(\bar P_T)}:=\{\bar P_T(f^2)\}^{1/2}$ and space $L^2(\bar P_T)$. Finally, define the pooled second-moment matrix $\bar\Sigma_T:=\E_{P^{(T)}}\!\left[\frac1T\sum_{t=1}^T\varphi_T(Z_t)\varphi_T(Z_t)^\top\right]$. 

In general, we use the bar notation to denote averaging over time and the distribution of trajectories. 
\begin{assum}
\label{eq: characteristic_feature}
We assume that the set $\{\varphi_T(z) \mid z \in \mathcal{Z}\}$ is linearly independent. 
\end{assum}
For $T\geq 1$, we define the sequence of statistical models $\mathcal{M}_T\subset \mathcal{M}_T^{\mathrm{np}}$ as 
\begin{align}
\label{def: stat_model_linear_well_spec}
    \mathcal{M}_T = \left\{P^{(T)} \in \mathcal{M}_T \mid \exists \beta_0 \in \mathbb R^{d_T} \;\text{s.t.}\;\forall z\in \mathcal{Z}, 
    \mathbb{E}_{P^{(T)}}[Y_t \mid Z_t = z] = \varphi_T(z)^{\top}\beta_0\right\}. 
\end{align}
In Eq.~\eqref{def: stat_model_linear_well_spec}, $\mathbb{E}_{P^{(T)}}[Y_t \mid Z_t = z]$ denotes the conditional expectation of the invariant factor $q_Y$ in the factorization of $P^{(T)}$ as in Eq.~\eqref{eq:model_factorization}. By Assumption~\ref{eq: characteristic_feature}, for $P^{(T)}\in \mathcal{M}_T$, there can only be one $\beta_0$ satisfying the condition in Eq.~\eqref{def: stat_model_linear_well_spec}. When there is no ambiguity as to the underlying distribution $P^{(T)}$, we denote the corresponding unique $\beta_0$ by $\beta_T$; we similarly define $h_T := \mathbb{E}_{P^{(T)}}[Y_t\mid Z_t = \cdot]$.

\begin{rem}[Why do we consider high dimensional setting?]
    We consider asymptotic normality in setting of correctly specified linear sieves with growing dimension as it is precisely the setting that reveals a gap between the plug-in OLS estimator and the one-step estimator in the i.i.d.\ setting. The plug-in OLS estimator exhibits a naive dependence on the ambient dimension $d$, and requires $d_T = o(T)$ for consistency. In contrast, the one-step estimator with regularized nuisances incur a second order remainder term of order $\frac{d_{\mathrm{eff}}(\lambda)}{T}$, where $d_{\mathrm{eff}}$ is the effective dimension associated with regularization strength $\lambda$, allowing for potentially much more aggressive scaling of dimensions. 
\end{rem}

We make following assumptions on the outcome model error.

\paragraph{Sequence of Target Parameters.}

Define the \textit{sequence} of functionals $\Psi_T : \mathcal{M}_T \to \mathbb{R}$, for any $T \geq 1$, 
\begin{equation}
\label{eq:linear_target}
\Psi_T(P^{(T)}) = \nu_T^\top \beta_T, 
\end{equation}

\paragraph{Riesz representation of the trajectory-level target.}

Although $\Psi_T$ is defined as a functional of the full trajectory law $P^{(T)}$,
it depends on $P^{(T)}$ only through the invariant conditional mean
$h_T(z) := \E_{P^{(T)}}[Y_t \mid Z_t = z]$.
To formalize this dependence, we leverage the time-averaged $L^2$ geometry $L^2(\bar{P}^{(T)})$ associated with $P^{(T)}$ defined in Eq.~\eqref{eq: l2_bar_pt_defn}.

\begin{assum}[Identification condition]
\label{ass: source}
$\nu_T \in \ker(\bar\Sigma_T)^\perp$.
\end{assum}
We define the function
$\bar\alpha_T : \mathcal Z \to \mathbb R$ by $\bar\alpha_T(z) := \nu_T^\top \bar\Sigma_T^{\dagger} \varphi_T(z)$, 
where $\bar\Sigma_T^{\dagger}$ denotes the Moore--Penrose pseudoinverse. We then define the linear functional
$\widetilde\Psi_T : L^2(\bar P^{(T)}) \to \mathbb R$ by
\begin{equation}
\widetilde\Psi_T(h)
:= \frac{1}{T}\sum_{t=1}^T
\E_{P^{(T)}}\!\left[\bar\alpha_T(Z_t)h(Z_t)\right].    
\end{equation}
We then have that for any $P^{(T)} \in \mathcal{M}_T$, $\Psi_T(P^{(T)}) = \widetilde\Psi_T(h_T)$, and that $\bar \alpha_T$ is the Riesz representer of $\widetilde\Psi_T$.

\section{Canonical Gradient}
\label{sec:canonical_grad}

In this section, we derive the canonical gradient of $\Psi_T$ for a given $T$. The following homoskesdasticity assumption simplifies the analysis.
\begin{assum}[Homoskedastic noise]
\label{ass: homoschedastic}
There exists $\sigma > 0$ independent of $T$ and $z$ such that, $\mathbb{E}_{P^{(T)}}[\varepsilon_t^2 \mid Z_t = z] = \sigma^2$, where $\varepsilon_t := Y_t - \mathbb{E}[Y_t\mid Z_t, \mathcal{F}_{t-1}]$.
\end{assum}
\begin{thm}[Canonical gradient]
\label{thm: sp_can_grad}
Under assumptions \ref{eq: characteristic_feature}-\ref{ass: homoschedastic}, $\Psi_T$ is pathwise differentiable at $P^{(T)}$ if, and only if, 
\begin{equation}
\label{eq:canonical_gradient}
D^{\ast}_T = D^*_T(\bar \alpha_T, h_T) \quad \text{where} \quad D_T^*(\alpha, h): \bar{o}_T  \mapsto \frac{1}{T}\sum_{t=1}^T
\alpha(z_t)\bigl\{y_t-h(z_t)\bigr\},
\end{equation}
has bounded $L^2(P^{(T)})$ norm (i.e. $\|D^{\ast}_T\|^2_{L^{2}(P^{(T)})} = \frac{\sigma^2}{T}\nu_T^{\top}\Bar{\Sigma}_T^{\dag}\nu_T < \infty$), in which case it is its canonical gradient.
\end{thm}

We refer the reader to Section~\ref{app:canonical_gradient} for the full proof. The above expression makes explicit the martingale structure of the canonical gradient under adaptive data collection and is the basis of our asymptotic analysis. Note that the canonical gradient makes appear the marginal average design matrix $\bar \Sigma_T$ through the definition of the Riesz representer $\bar \alpha_T$, in other words, $\bar\Sigma_T$ arises from averaging in $T$ \textit{longitudinally} and at the population level across trajectories \textit{cross-sectionally}. This is to be contrasted with the variance-stabilized doubly robust scores used in \citep{bibaut2021post} for instance, that do not coincide with the above canonical gradient, which suggests that their corresponding stabilized AIPW is likely not efficient. 

\section{Directional stability}

In this section we introduce the notion of directional stability which is the cornerstone of our asymptotic analysis. Let
\begin{align}
   \widehat \alpha_{T, \lambda_\alpha} : z \mapsto \nu^\top \left(\widehat \Sigma_T + \lambda_{\alpha} I_{d_T} \right)^{-1}\varphi_T(z), \label{eq:riesz_estimator}
\end{align}
where $\widehat \Sigma_T := \frac{1}{T} \sum_t \varphi_T(Z_t) \varphi_T(Z_t)^\top$.
\begin{defn}[Directional stability]\label{def:dir_stab}
    We say that the design is directionally stable w.r.t. $(\Psi_T)_{T\geq 1}$ if there exists a deterministic sequence of functions $(\widetilde \alpha_T)_{T\geq 1}$ of the form $\widetilde \alpha_T : z \mapsto \nu_T^\top \widetilde \Sigma_T^{-1} \varphi_T(z)$ for a sequence of positive definite matrices $(\widetilde \Sigma_T)_{T \geq 1}$ such that 
    \begin{align}
        \|  \widehat \alpha_{T, \lambda_\alpha} - \widetilde \alpha_T \|_{L^2(P_T)} = o_{P^{(T)}}(1).
    \end{align}
\end{defn}
We note that when $\lambda_\alpha = o(\lambda_{\min}(\widetilde \Sigma_T))$, where $\lambda_{\min}(\widetilde \Sigma_T)$ is the smallest eigenvalue of $\widetilde \Sigma_T$, then directional stability is always implied by the classical full-matrix notion of stability \citep{laiweiaos1982}, which requires that $\| \widetilde \Sigma_T^{-1} \widehat \Sigma_T - I_{d_T} \|_{\mathrm{op}} = o_{P^{(T)}}(1)$. 

\begin{rem}
    Note that in the above definition $\widetilde \Sigma_T$ need not relate to $\bar \Sigma_T$ from the canonical gradient. Should $\bar \Sigma_T$ be a valid choice for $\widetilde \Sigma_T$, then $(\widehat \alpha_{T, \lambda_\alpha})$ is a consistent estimator of the Riesz representer sequence $(\bar \alpha_T)$. We will see further down how consistent estimation of the Riesz representer sequence is a necessary condition for efficiency of the one-step. 
\end{rem}

\section{One-Step Estimator}
\label{sec:one_step_estimator}

We construct a one-step estimator by plugging in the ridge regularized Riesz-representer-like estimator $\widehat{\alpha}_{T, \lambda_\alpha}$ from \eqref{eq:riesz_estimator} as in \citep{chernozhukov2022automatic}, and the following ridge-regularized estimator of the outcome regression function: 
\[
\widehat h_{T,\lambda_h}(z)
:= \varphi_T(z)^\top \widehat\beta_{T,\lambda},
\qquad \text{where} \qquad
\widehat\beta_{T,\lambda_h}
:= \widehat\Sigma_{T,\lambda_h}^{-1}\widehat\Sigma_{T,ZY},
\]
with
\[
\widehat\Sigma_{T,ZY}
:= \frac{1}{T}\sum_{t=1}^T \varphi_T(Z_t)\,Y_t .
\]
We then construct a one-step estimator of $\Psi(P^{(T)})$ as follows:
\[
\widehat \Psi
:= \Psi(\widehat h_{T, \lambda_h})
+ \bar P_T \left[
\widehat\alpha_{T,\lambda_\alpha}(Z)\,\{Y-\widehat h_{T,\lambda_h}(Z)\}
\right].
\]

\begin{rem}[Equivalence to Undersmoothed Plug-in Estimator]
In the i.i.d. case, \citet{bruns2025augmented} show that when both the outcome model and the Riesz
representer are estimated by (kernel) ridge regression (“double ridge”), the resulting augmented estimator
is numerically identical in finite samples to a single (kernel) ridge regression plug-in estimator with a
smaller effective penalty parameter. Due to this numerical equivalence, an undersmoothed plug-in ridge estimator inherits the theoretical properties we established for the one-step estimator. 
\end{rem}

\section{Asymptotic Analysis}
\label{sec: asymptotic_analysis}

In this section we provide an asymptotic analysis of our one step estimator using our definition of directional stability. 

Let $\widetilde \sigma_T := \sigma (\nu_T^\top \widetilde \Sigma_T^{-1} \nu_T)^{1/2}$, which is the quantity appearing in the source condition of exponent 1 of the Riesz representer w.r.t $\widetilde \Sigma_T$. An alternative representation is $\widetilde \sigma_T = (\widetilde P^{(T)} \{\widetilde \alpha_T (y - h_T) \}^2)^{1/2}$, where $\widetilde P^{(T)}$ is the operator defined for any fixed function  $f:  (z, \varepsilon) \mapsto a_1 \varphi(z) \varphi(z)^\top + a_2 \varphi(z) \varepsilon$ by $\widetilde P^{(T)} f := a_1 \Sigma_T.$ This alternative representation highlights the interpretation of $\widetilde \sigma_T^2$ as a variance-like object of the generic term $\widetilde \alpha_T (y - h_T)$ of $D^*(\widetilde \alpha_T, h_T)$. In general $\widetilde \sigma_T$ needs not coincide with 
\begin{align*}
    \bar{\sigma}_T := 
\sigma\sqrt{\nu^{\top}_T\bar{\Sigma}_T^{\dag}\nu_T} = \|\bar{\alpha}_T(y - h_T)\|_{L^2(\bar{P}^{(T)})} = \sqrt{T}\|D^{\ast}_T\|_{L^2(P^{(T)})}
\end{align*}
(see Eq.~\eqref{eq: sigma_bar_T_nu_expression} for a derivation of this equality). It instead plays an analogous role for adaptive data collection under stable design. In general $\widetilde \sigma_T$ does not converge as a function of $T$, as can be checked for example in the case of LinUCB from the proof of Proposition~\ref{prop:linucb_stab}.

\begin{assum}[Directional stability]
\label{ass: directional_stab_vme}
    It holds that $\frac{\sigma^2}{\widetilde{\sigma}_T^2}\nu_T^{T}\widetilde{\Sigma}_T^{-1}(\widehat{\Sigma}_T - \widetilde{\Sigma}_T)\widetilde{\Sigma}_T^{-1}\nu_T \overset{p}{\rightarrow} 0$. 
\end{assum}
\begin{assum}[Lindeberg condition]
\label{ass: lindeberg_vme}
    For all $\epsilon > 0$, we have
\begin{align*}
    \sum_{t=1}^{T}\mathbb{E}\left[\frac{1}{T\widetilde{\sigma}_T^2}\widetilde{\alpha}_T(Z_t)^2 \varepsilon_t^2 1\left[\frac{\widetilde{\alpha}_T(Z_t)^2\varepsilon_t^2}{T\widetilde{\sigma}_T^2} > \epsilon\right]\mid \mathcal{F}_{t-1}\right] \overset{p}{\rightarrow} 0.
\end{align*}
\end{assum}
\begin{assum}[Gaussian noise]\label{ass:gauss_noise}
There exists $\sigma\in(0,\infty)$ such that, conditional on $Z_{1:T}:=(Z_1,\dots,Z_T)$,
the noise variables $(\varepsilon_t)_{t=1}^T$ are independent and satisfy
\[
(\varepsilon_1,\dots,\varepsilon_T)\mid Z_{1:T}\ \sim\ \mathcal{N}(0,\sigma^2 I_T).
\]
Equivalently, $\varepsilon_t\mid Z_{1:T}\stackrel{\text{i.i.d.}}{\sim}\mathcal{N}(0,\sigma^2)$.
\end{assum}
\begin{thm}[von Mises expansion and asymptotic normality]\label{thm:vme-and-an}
Under assumption~\ref{ass:gauss_noise}, the following von-Mises expansion holds provided the quantities in it are well-defined and bounded:
\begin{align*}
    \widehat \Psi_T - \Psi_T =& (P_T - \widetilde P^{(T)})\{\alpha_T (Y - h_T)\} + O\left(R_T \right)
\end{align*}
with
\begin{align}
    R_T := \left\|\widehat{\alpha}_{T, \lambda_\alpha} - \widetilde \alpha_T \right\|_{L^2(P_T)} \left( \left\|\widehat{h}_{T,\lambda_h} - h_T\right\|_{L^2(P_T)}  + \frac{1}{\sqrt{T}}\right) + \frac{1}{T} \nu_T^\top \widehat \Sigma_{T,\lambda_\alpha}^{-1} (\widehat{\beta}_{T,\lambda_h} - \beta_T).
\end{align}
If we have that $R_T = o_P\left(\frac{\widetilde \sigma_T}{\sqrt{T}}\right)$ and assumptions~\ref{ass: directional_stab_vme}-\ref{ass: lindeberg_vme} hold, then
it further holds that 
\begin{align*}
    \sqrt{T / \widetilde \sigma_T} \left( \widehat \Psi_T - \Psi_T  \right) \xrightarrow{d} \mathcal{N}(0,1).
\end{align*}
\end{thm}
We refer the reader to Section~\ref{appendix:one_step_analysis} for a full proof.
\begin{proof}[Proof sketch of Theorem~\ref{thm:vme-and-an}]
Write the one-step error as the explicit von Mises expansion
\begin{align*}
    \widehat\psi-\Psi_T
&=\underbrace{(P_T-\widetilde P^{(T)})\widetilde\alpha_T(Y-h_T)}_{\text{(A)}}
+\underbrace{(P_T-\widetilde P^{(T)})\widetilde\alpha_T(h_T-\widehat h_{T,\lambda})}_{\text{(B)}}\\
&+\underbrace{(P_T-\widetilde P^{(T)})(\widehat\alpha_T-\widetilde\alpha_T)\varepsilon}_{\text{(C)}}
-\underbrace{P_T(\widehat\alpha_T-\widetilde\alpha_T)(\widehat h_{T,\lambda}-h_T)}_{\text{(D)}}.
\end{align*}
For (A), set $X_{T,t}:=\widetilde\alpha_T(Z_t)\varepsilon_t/(\widetilde\sigma_T\sqrt T)$; then $(X_{T,t})_{t\le T}$ is a martingale difference array.
Assumption~\ref{ass: directional_stab_vme} gives the conditional quadratic variation
$\sum_{t\le T}\E[X_{T,t}^2\mid \mathcal F_{t-1}]\to_p 1$, and Assumption~\ref{ass: lindeberg_vme} gives the conditional Lindeberg condition, so the martingale CLT (Lemma~\ref{lem:mclt}) yields
$\sqrt T\,\text{(A)}/\widetilde\sigma_T\Rightarrow \mathcal N(0,1)$.
For the remainder term, (C) is a conditional centered Gaussian with variance
$\sigma^2\|\widehat\alpha_T-\widetilde\alpha_T\|_{L_2(P_T)}^2/T$, hence
$\text{(C)}=O_p(\|\widehat\alpha_T-\widetilde\alpha_T\|_{L_2(P_T)}/\sqrt T)$.
Moreover, by Cauchy--Schwarz,
\begin{align*}
    |\text{(D)}|\le \|\widehat\alpha_T-\widetilde\alpha_T\|_{L_2(P_T)}\,
\|\widehat h_{T,\lambda}-h_T\|_{L_2(P_T)}, 
\end{align*}
and (B) is bounded by the second order bound plus an additional bias term $\frac{1}{T} \nu^\top \widehat \Sigma_{T,\lambda_\alpha}^{-1} (\widehat{\beta}_{T,\lambda_h} - \beta_T)$.

\end{proof}
\paragraph{Instantiation under LinUCB.}
We consider LinUCB as presented by \citet[Chapter 19]{lattimore2020}, with a $d_T$-dimensional feature map. We use $\gamma$ for their $\beta$ (the ``exploration bonus'' factor). We define the deterministic matrix
\begin{align*}
    \widetilde \Sigma_T := P_\star + (T d_T)^{1/4} \gamma^{-1/2} P_\perp, 
\end{align*}
where $P_\star = \beta_T \beta_T^{\top}/ \|\beta_T\|_2^2$ and $P_{\perp} = \mathrm{Id}_{\mathcal{H}} - P_{\star}$.

\begin{assum}
\label{ass: bounded_features}
    There exists $L>0$ such that for all $T\geq 1$ and for all $z\in \mathcal{Z}$, we have $\|\varphi_T(z)\|\leq L$. 
\end{assum}
We require a tail assumption with same conditional variance scale as assumption~\ref{ass: homoschedastic}. 
\begin{assum}
    \label{ass: sub_gaussian_noise}
    The sequence $(\epsilon)_{t=1}^{T}$ is an $(\mathcal{F}_t)$-adapted MDS with conditional variance $\sigma^2$, and satisfies: for all $t\geq 1$ and $\lambda \in \mathbb{R}$, 
    \begin{align*}
        \mathbb{E}[\exp(\lambda \epsilon_t)\mid \mathcal{F}_{t-1}] \leq \exp\left(\frac{1}{2}\sigma^2\lambda^2\right).
    \end{align*}
\end{assum}
Assumption~\ref{ass: sub_gaussian_noise} plays a key role in the self-normalized inequalities of \citet{NIPS2011_e1d5be1c} and the non-asymptotic analysis of \citet{fan2025statisticalinferenceadaptivesampling}. The following proposition is a direct consequence of \citep[Theorem 2]{NIPS2011_e1d5be1c} along with \citep[Proposition 2]{jezequel2019}. Such bounds cannot be improved for a general bandit algorithm \citep{pmlr-v195-lattimore23b}. 
\begin{prop}
Under assumption~\ref{ass: bounded_features}-\ref{ass: sub_gaussian_noise}, it holds with probability $\geq 1-\delta$, for all $T\geq 0$, that
    \begin{align*}
        \left\|\widehat h_{T,\lambda_h} - h_T\right\|_{L^2(P_T)} \lesssim \sqrt{\frac{ d_{\mathrm{eff}}(\lambda_h, T)}{T}} + \sqrt{\lambda_h} \| \beta_0 \|_2.
    \end{align*}
    where $d_{\mathrm{eff}}(\lambda_h, T) :=  \Tr\left((\lambda_h I_{d_T}+ \widehat \Sigma_T)^{-1} \widehat \Sigma_T\right)$ is the effective dimension.
\end{prop}
The notion $d_{\mathrm{eff}}(\lambda, T)$ is related to the Bayesian information gain \citep{krause2010}. A reader familiar with the regression literature \citep{caponnetto2007optimal, fischer2020sobolev} may be more accustomed to effective dimension as defined in terms of population covariance matrix. The term $\sqrt{\lambda} \| \beta_0 \|_2$ is a trivial upper bound on the bias term, and can be improved under source conditions on $\beta_0$ with respect to $\widehat{\Sigma}_T$. The following directional stability proposition is a consequence of
\citep{fan2025statisticalinferenceadaptivesampling}. 
\begin{assum}[Large exploration]
\label{ass: large_exploration}
    The exploration bonus $\gamma_T$ for horizon $T$ 
satisfies
\begin{align*}
    \gamma_T \gtrsim d_T^2 (\sigma\sqrt{d_T + \log\log T} + 1). 
\end{align*}
\end{assum}
\begin{prop}\label{prop:linucb_stab} Under assumptions~\ref{ass: bounded_features}-\ref{ass: sub_gaussian_noise}-\ref{ass: large_exploration}, suppose that $\|\varphi_T(Z_t)\| = 1$ for every $t$, and  that
\begin{align*}
\varepsilon_{\mathrm{bulk}} := d_T \left( \frac{\gamma_T^8}{T} \right)^{\frac{d_T+1}{d_T-1}} + \frac{d_T^{1/4}}{\sqrt{\gamma_T}} = o(1), \quad d_T= o(T), \quad \gamma_T \sqrt{d_T / T} = o(1), \quad \text{and} \quad \gamma_T^{-1} = o(1).
\end{align*}
For $\lambda_\alpha = 1/T$, omitting polylogs, we have the following bound in probability, with proof given in Section~\ref{appendix:stability_rate}. 
    \begin{align}
\frac{\| \widehat \alpha_{T, \lambda_\alpha} - \widetilde \alpha_T \|_{L^2(P_T)}}{ \widetilde \sigma_T} &= O_{P^{(T)}} \left( \left(\frac{d_T}{\gamma_T} + \sqrt{\frac{d_T}{T}} \gamma_T \right)  \frac{\| P_\star \nu_T\|}{\| P_\star \nu_T\| + \|P_{\perp} \nu_T \| (T d_T)^{1/4} / \sqrt{\gamma_T} } \tag{$\parallel$}\label{eq: directional_stab_aligned} \right.\\
&\left.+ \left( \sqrt{\frac{d_T}{T}} + \frac{d_T}{\gamma_T} + \varepsilon_{\mathrm{bulk}} \right)\frac{\|P_{\perp} \nu_T \| (T d_T)^{1/4} / \sqrt{\gamma_T} }{\| P_\star \nu_T\| + \|P_{\perp} \nu_T \| (T d_T)^{1/4} / \sqrt{\gamma_T} } \tag{$\perp$} \right)\label{eq: directional_stab_perp}
\end{align}
\end{prop}
Here Eq.~\eqref{eq: directional_stab_aligned} reflects the contribution from the mass of $\nu$ aligned with the true signal direction $h_T$, and Eq.~\eqref{eq: directional_stab_perp} reflects the contribution from the mass of $\nu$ orthogonal to the true signal direction. This is the key distinction between adaptive data collection and data collected under i.i.d. design, as the bandit is regret-minimizing and learns to collect less data in the directions corresponding to $P_{\perp}$. Proposition~\ref{prop:linucb_stab} characterizes when asymptotic normality and regret minimization can  simultaneously be achieved, by optimally tuning the parameters $\gamma, d_T$ as a function of $T$ in the high dimensional regime such that both Eq.~\eqref{eq: directional_stab_aligned} and Eq.~\eqref{eq: directional_stab_perp} are $o(1)$ with respect to $T$, while the expected regret remains sublinear and is as small as possible.

\begin{rem}
    The bound in Proposition~\ref{prop:linucb_stab} makes appear the term $\sqrt{\frac{d_T}{T}}$, as we regularize $\widehat{\alpha}_T$ with ridge penalty $\frac{1}{T}$ or faster. We conjecture that it may be possible to derive a generalization of the above proposition to an arbitrary regularization schedule, and instead make appear the term $\sqrt{\frac{d_{\mathrm{eff}}(\lambda, T)}{T}}$, allowing for growth of $d$ that is superlinear in $T$. In other words, we expect that regularization improves directional stability, as it does in the i.i.d. setting (see e.g., \citealp{bruns2025augmented}). Under this conjecture, the second order remainder term in the von Mises expansion can be $o_P(\widetilde{\sigma}_T/\sqrt{T})$ under much milder restrictions on growth of $d_T$ relative to $T$, yielding asymptotic normality under directional stability in high dimensional regimes where the OLS fails to be normal.   
\end{rem}

\section{Efficiency Theory}
\label{sec: efficiency_theory}

We consider efficiency for the sequence of statistical experiments
$\{\mathcal{M}_T\}_{T \ge 1}$, where both the target parameter $\Psi_T$ and its
canonical gradient $D_T^*$ depend on the experimental horizon $T$. 

\begin{assum}[Pathwise differentiability]
\label{ass: pathwise_differentiability}
    For each $T\in \mathbb{N}$, the parameter $\Psi_T : \mathcal{M}_T \to \mathbb{R}$ is pathwise differentiable at $P^{(T)} \in \mathcal{M}_T$ with canonical gradient $D^{\ast}_T \in L^2(P^{(T)})$. 
\end{assum}

To formulate local asymptotic normality (LAN) and efficiency when $\sqrt{T}$-rates are not attainable, we state a notion of least favorable submodels with $T$-specific Fisher information \citep{vanderVaart1998asymptotic}.

\begin{defn}[Sequence of Least favorable submodels]\label{def:lfs}
Let $\Psi_T:\mathcal M_T\to\mathbb R$ be pathwise differentiable at
$P^{(T)}$ with canonical gradient $D_T^*\in L^2(P^{(T)})$, and let $\bar{\sigma}_T$ be as defined in Section~\ref{sec: asymptotic_analysis}. A sequence of one--dimensional submodels
$\{P^{(T)}_\eta:|\eta|\le\delta\}\subset\mathcal M_T$ with
$P^{(T)}_0=P^{(T)}$ is said to be \emph{least favorable} for the sequence of tuples $(\Psi_T, \mathcal{M}_T, P^{(T)})_{T\geq 1}$ if 
\begin{enumerate}
    \item For each fixed $T$, the submodel $\{p_\eta^{(T)}\mid |\eta|\leq \delta\}$ is QMD at $\eta = 0$ with score function 
    \begin{equation}
\partial_\eta \log p^{(T)}_\eta(\bar O_T)\big|_{\eta=0}
=
\,D_T^*(\bar O_T),
\label{eq:score_submodel}
\end{equation}
    \item The QMD remainder vanishes, for each fixed $\epsilon \in \mathbb{R}$, as $T\to \infty$, we have
    \begin{equation}
    \label{eq: QMD_remainder_vanishes}
\int\!\Big(
\sqrt{dP^{(T)}_{\eta}}
-\sqrt{dP^{(T)}}
-\tfrac{\eta}{2}
D_T^*\sqrt{dP^{(T)}}
\Big)^2
=o(\eta^2),
\qquad \eta\to0.
\end{equation}
\end{enumerate}
\end{defn}

Note that the Fisher information in the parameterization defined by \eqref{eq:score_submodel} is $(\bar P^{(T)} \{(D^*_T)^2\})^{-1}$ which diverges as $T \to \infty$: this is because it is a whole-trajectory Fisher information, while typical analyses in the i.i.d. setting work with a per-sample notion of Fisher information. In  Appendix \ref{appendix:least_favorable_submodel} we show the existence of such a submodel. Theorem~\ref{thm:lan} below identifies the local asymptotic structure of the
adaptive experiment along least favorable submodels. 

\begin{assum}[Directional stability]
\label{ass: lan_stability}
Let $\sigma$ be as in Assumption~\ref{ass: homoschedastic}. We assume that
    \begin{align*}
        \frac{\sigma^2}{\bar{\sigma}_T^2} \nu^\top_T \bar{\Sigma}_T^{\dag}\widehat{\Sigma}_T\bar{\Sigma}_T^{\dag}\nu_T \overset{p}{\rightarrow} 1 \iff  \frac{\sigma^2}{\bar{\sigma}_T^2} \nu^\top_T \bar{\Sigma}_T^{\dag}(\widehat{\Sigma}_T - \bar{\Sigma}_T)\bar{\Sigma}_T^{\dag}\nu_T \overset{p}{\rightarrow} 0.
    \end{align*}
\end{assum}
Notice that this is the same notion of directional stability as introduced earlier, except that we now require the deterministic stabilizing sequence to be $\bar \Sigma_T$.
\begin{assum}[Lindeberg condition]
\label{ass: conditional_lind_eff}
    We assume that, for all $\epsilon > 0$ \begin{align*}
        \sum_{t=1}\mathbb{E}\left[\eta^2 s_T(Z_t, Y_t)^2 1[|\eta s_T(Z_t, Y_t)| > \epsilon]\mid \mathcal{F}_{T,t-1}\right] \overset{p}{\rightarrow} 0,
    \end{align*}
    where $s_T(z,y) := \frac{1}{T}\bar{\alpha}_T(z)(y - h_T(z))$ and $\eta = \epsilon \bar{\sigma}_T^{-1}T^{\frac{1}{2}}$. 
\end{assum}
\begin{thm}[Local asymptotic normality]\label{thm:lan}
We require assumptions~\ref{ass: lan_stability}-\ref{ass: conditional_lind_eff}, and that there exists a sequence of least favorable submodels $\{P^{(T)}_{\eta} : |\eta|\leq \delta\}$ in the sense of
Definition~\ref{def:lfs}, and define $\Delta_T:= \frac{\sqrt{T}}{\bar\sigma_T}\,D_T^*(\bar O_T)$. Suppose that $I_T = \frac{\bar{\sigma}_T^2}{T} \to \infty$ as $T\to \infty$. Then, for each fixed $\epsilon\in\mathbb R$,
\begin{equation}
\log\frac{dP^{(T)}_{\epsilon\sqrt T/\bar \sigma_T}}{dP^{(T)}}(\bar O_T)
=
\epsilon\,\Delta_T-\frac{\epsilon^2}{2}+o_{P^{(T)}}(1),
\label{eq:lan_adaptive}
\end{equation}
and we have $\epsilon \Delta_T \overset{d}{\rightarrow} \mathcal{N}(0, \epsilon^2)$. 
\end{thm}

\begin{proof}[Proof sketch of Theorem~\ref{thm:lan}]
Fix $\epsilon>0$ and set $I_T:=\bar{\sigma}_T^2/T$, $\eta:=\epsilon I_T^{-1/2}$.
Define
\[
W_t:=2\Big(\sqrt{q_{Y,\eta}(Y_t\mid Z_t)/q_Y(Y_t\mid Z_t)}-1\Big),
\qquad
\log\frac{p^{(T)}_\eta(\bar O_T)}{p^{(T)}(\bar O_T)}
=2\sum_{t=1}^T \log\Big(1+\tfrac12 W_t\Big).
\]
In Lemma~\ref{lem:lch_qmd} in the Appendix, we explicitly construct a sequence of least favorable submodels via perturbing the outcome likelihood $q_{Y, \eta}$, such that for remainder term $R_{t, \eta}$ satisfying $T\,\E[R_{t,\eta}^2]=o(1)$, we have $W_t=\frac{\eta}{T}\,\bar\alpha_T(Z_t)\{Y_t-h_T(Z_t)\}+R_{t,\eta}(Y_t,Z_t)$, 
which implies 
\begin{align*}
    \sum_{t=1}^T W_t
=
\frac{\eta}{T}\sum_{t=1}^T \bar\alpha_T(Z_t)\{Y_t-h_T(Z_t)\}
-\frac{\epsilon^2}{4}
+o_p(1). 
\end{align*}
Using $\log(1+x)=x-\frac{x^2}{2}+x^2R(x)$ with $R(x)\to0$ as $x\to0$, we have
\[
\log\frac{p^{(T)}_\eta(\bar{O}_T)}{p^{(T)}(\bar{O}_T)}
=
\sum_{t=1}^T W_t-\frac14\sum_{t=1}^T W_t^2+\frac12\sum_{t=1}^T W_t^2R(W_t).
\]
The term $\frac{\eta}{T}\,\bar\alpha_T(Z_t)\{Y_t-h_T(Z_t)\} \Rightarrow\ \mathcal N(0,\epsilon^2)$ by Assumptions~\ref{ass: lan_stability} and \ref{ass: conditional_lind_eff}, which verify the assumptions to apply martingale central limit theorem \citep[Theorem 3.2]{hall1980}. We show the quadratic term satisfies $\sum_{t=1}^T W_t^2\to_p \epsilon^2$, and $\sum_{t=1}^T W_t^2R(W_t)=o_p(1)$ since $\max_t|W_t|=o_p(1)$. Combining the above yields $\log\frac{p^{(T)}_\eta(\bar O_T)}{p^{(T)}(\bar O_T)}
\ \Rightarrow\ \mathcal N\!\left(-\frac{\epsilon^2}{2},\,\epsilon^2\right)$. 
\end{proof}
A full proof is deferred to Appendix~\ref{subsection: LAN}. As a consequence, for any sequence of statistics $\phi_T(\bar O_T)$,
the joint limit behavior of $(\phi_T,\Delta_T)$ under $P^{(T)}$
determines the entire family of asymptotic laws of $\phi_T$ under the
local experiments
$\{P^{(T)}_{\epsilon\sqrt T/\bar \sigma_T}:\epsilon\in\mathbb R\}$.
By the matching theorem \citep[Theorem 7.10]{vanderVaart1998asymptotic}, these limit laws can be realized by a single statistic acting on the Gaussian shift experiment
$X_\epsilon\sim\mathcal N(\epsilon,1)$.

We now introduce a notion of regularity for estimator sequences
$\{\widehat\psi_T\}_{T\ge1}$, requiring that their asymptotic distribution,
when centered at the local target
$\Psi_T(P^{(T)}_{\epsilon\sqrt T/\bar \sigma_T})$,
be invariant with respect to $\epsilon$.

\begin{defn}[Regularity along least favorable submodels]\label{def:regularity}
Let $\{P^{(T)}_{\epsilon\sqrt T/\bar \sigma_T}:\epsilon\in\mathbb R\}$ be a
least favorable submodel for $\Psi_T$ in the sense of
Definition~\ref{def:lfs}. A sequence of estimators
$\{\widehat\Psi_T\}_{T\ge1}$ is said to be \emph{regular} along this sequence
if there exists a probability law $L$ such that, for every fixed
$\epsilon\in\mathbb R$,
\begin{equation}
\frac{\sqrt{T}}{\bar\sigma_T}
\Big(\widehat\Psi_T-\Psi_T(P^{(T)}_{\epsilon\sqrt T/\bar \sigma_T})\Big)
\ \Rightarrow\ L
\quad\text{under } P^{(T)}_{\epsilon\sqrt T/\bar \sigma_T},
\end{equation}
with the same limit law $L$ for all $\epsilon$.
\end{defn}

\begin{rem}
    The efficiency notion developed here is relative to the class of regular
estimators defined in Definition~\ref{def:regularity}. Regularity requires
that the asymptotic distribution of the estimator, when centered at the
local target along least favorable submodels, be invariant under
$\sqrt{T}/\bar \sigma_T$–scale perturbations.
Unlike the i.i.d.\ setting, the natural local scale of the experiment
depends on $T$ through the factor $\bar \sigma_T$.
Consequently, local perturbations are not comparable across different
horizons $T$ without rescaling. The normalization by $\bar \sigma_T/ \sqrt{T}$
identifies a common local parameterization under which the localized
experiments admit a homogeneous Gaussian shift limit.
 \end{rem}

This regularity condition translates, in the Gaussian shift experiment, into
equivariance-in-law and forms the basis for the convolution and
efficiency results that follow.

\begin{thm}[Convolution theorem]
\label{thm: conv_thm}
Under Assumption~\ref{ass: pathwise_differentiability}-\ref{ass: lan_stability}-\ref{ass: conditional_lind_eff}. 
Suppose that $\{\widehat{\Psi}_T\}_{T\geq 1}$ is regular along a least
favorable submodel in the sense of
Definition~\ref{def:regularity}. Then, under $P^{(T)}$,
\begin{equation}
\frac{\sqrt{T}}{\bar\sigma_T}\big(\widehat\Psi_T-\Psi_T(P^{(T)})\big)
\ \Rightarrow\ L,
\end{equation}
and there exists a probability measure $M$ on $\mathbb R$ such that $L \;=\; \mathcal N(0,1)\ast M$. 
In particular, if $L$ has variance $\sigma^2$, then $\sigma^2\ge 1$, with
equality if and only if $M=\delta_0$ (equivalently,
$L=\mathcal N(0,1)$).
\end{thm}

Consequently, $\widehat\Psi_T$ is asymptotically efficient among
regular estimators if and only if
\begin{equation}
\frac{\sqrt{T}}{\bar\sigma_T}
\big(\widehat\Psi_T-\Psi_T(P^{(T)})\big)
\ \Rightarrow\ \mathcal N(0,1).
\end{equation}

In particular, any asymptotically linear estimator with influence function
$D_T^*$ attains this limit; hence the one--step estimator constructed in the previous section is asymptotically efficient.

\begin{rem}
    An asymptotic efficient sequence has $\sqrt{T}$-scaled rate $\sqrt{\nu^\top \bar \Sigma_T^{\dag} \nu}$ whenever $D_{T,Y}^*$ dominates $D_{T,X}^*$, which is typically the case. Meanwhile, we have proved asymptotic normality at $\sqrt{T}$-scaled rate $\sqrt{\nu^\top \Sigma_T^{-1} \nu}$. For our one-step to be efficient, these need to be asymptotically equivalent.
\end{rem}

\begin{rem}
    While the notion of efficiency might seem ad-hoc, instantiating it in the i.i.d. setting recovers the ``usual'' semiparametric efficiency bound. In other words, in the i.i.d. case, competing against estimators that are only regular along the sequence of least favorable submodels is as hard as competing against all regular estimators.
\end{rem}

\section{Discussions}

This work shows that \emph{directional stability} is sufficient for valid and
efficient inference in adaptive experiments. Under this mild condition, adaptively collected data can be treated with estimators that are algebraically equivalent to i.i.d estimators and asymptotically efficient.

Several directions for future work remain open. First,  the analysis may be extended to settings where the working model for the reward function is misspecified, and  our target parameter may be defined as a nonparametric M-estimand to further anchor the robustness of the proposed framework. Second, our results
currently focus on $\sqrt{d_T/T}$ rates induced by specific ridge regularization
schedules; extending the theory to more general regularization regimes, in
which complexity is governed by the effective dimension
$d_{\mathrm{eff}}(\lambda,T)$, is an important next step. A final direction is to study debiased inference after model selection in adaptive settings,
building on recent advances such as
\citet{vanderlaan2023adaptivedebiasedmachinelearning}.

\newpage

\bibliographystyle{plainnat}
\bibliography{references}   

\newpage
\section*{Appendix}
\input{appendix/main}
\newpage

\end{document}

%% file: appendix/main.tex
This appendix is organized as follows: 

\begin{itemize}[nosep, label={--}]
    \item Appendix~\ref{app:canonical_gradient}: We provide proofs for pathwise differentiability and canonical gradient. 
    \item Appendix~\ref{appendix:one_step_analysis}: Asymptotic analysis of the one-step estimator in the high dimensional regime.
    \item Appendix~\ref{appendix:stability_rate}: Upper Bound on Stability Rate under LinUCB Sampling. 
    \item Appendix~\ref{app:efficiency}: Proof of local asymptotic normality, and convolution theorem. 
\end{itemize}

\section{Canonical gradient}
\label{app:canonical_gradient}

\paragraph{Parameter and notation.}
Let $Z_t := (X_t,A_t)$ and write the conditional mean and residual as
\[
h_T(z)\ :=\ \E_{P^{(T)}}[Y_t\mid Z_t=z]\ =\ \varphi_T(z)^\top \beta_T,
\qquad
\varepsilon_t\ :=\ Y_t-h_T(Z_t). 
\]
We define the pooled $(X,A)$ marginal induced by the (possibly adaptive) logging sequence:
\[
\bar g_t(a\mid x)\ :=\ \int g_t(a\mid x,\bar o_{t-1})\,dP^{(T)}(\bar o_{t-1}),
\qquad
\bar g(a\mid x)\ :=\ \frac{1}{T}\sum_{t=1}^T \bar g_t(a\mid x).
\]
Equivalently, the pooled marginal density/mass of $Z$ with respect to $\mu_X\otimes \nu_{\mathcal A}$ is
\[
\bar h(a,x)\ :=\ q_X(x)\bar g(a\mid x).
\]
The pooled covariance matrix $\bar{\Sigma}_T$ can be written as 
\begin{align*}
    \bar{\Sigma}_T = \sum_{a\in \mathcal{A}}\int_{\mathcal{X}}\varphi(x,a)\varphi(x,a)^{\top}\,\bar{h}(a,x)\;\mathrm{d}x. 
\end{align*}
The target functional evaluated at $P^{(T)}$ can be written as
\[
\Psi_T(P^{(T)}) = \nu_T^{\top} \beta_T.
\]

We now present the proof of Theorem~\ref{thm: sp_can_grad}. 
\begin{proof}
Fix $P^{(T)} \in \mathcal{M}_T$, where 
\begin{align*}
    \frac{dP^{(T)}}{d \mu^{(T)}}(\bar{o}_T) = \prod_{t=1}^{T}q_X(x_t)g_t(a_t\mid x_t, \bar{o}_{t-1})q_{Y}(y_t\mid a_t, x_t). 
\end{align*}
We define
\begin{align}
\label{def: stat_model_linear_well_spec_res}
    \mathcal{M}_T^{\mathrm{res}} = \mathcal{M}_T \cap \left\{    \frac{dP^{(T)}}{d \mu^{(T)}}(\bar{o}_T) = \prod_{t=1}^{T}q_X(x_t)g_t(a_t\mid x_t, \bar{o}_{t-1})q_{Y'}(y_t\mid a_t, x_t)\mid q_{Y'}\right\},
\end{align}
as the restricted model of $\mathcal{M}_T$ where $q_Y$ is allowed to vary, and $q_X$, $g_t$, $1\leq t\leq T$ are fixed at the corresponding factors in the fixed $P^{(T)}$. 

We prove the theorem by (i) identifying the tangent space of $\mathcal{M}_T^{\mathrm{res}}$, (ii) computing the pathwise derivative of $\Psi_T$ with respect to $\mathcal{M}_T^{\mathrm{res}}$, (iii) showing that the pathwise derivative lies in the tangent space, hence it is the canonical gradient.  Since $\Psi_T$ only depends on $q_{Y}$, the canonical gradient is agnostic to if $g_t$ or $q_X$ is known. Therefore, the canonical gradient of $\Psi_T$ with respect to $\mathcal{M}_T^{\mathrm{res}}$ coincides with that with respect to $\mathcal{M}_T$. 

\paragraph{Step 1: Tangent space of $\mathcal{M}_T^{\mathrm{res}}$. }
Consider a one-dimensional submodel
$\{P^{(T)}_\epsilon:\epsilon\in(-\epsilon_0,\epsilon_0)\} \subseteq \mathcal{M}^{\mathrm{res}}_T$ through $P^{(T)}$. We assume $\epsilon \mapsto q_{Y,\epsilon}$ is differentiable in quadratic mean (DQM) at $\epsilon = 0$ with (conditional) score 
\[
s(y,a,x)\ :=\ \left.\frac{d}{d\epsilon}\log q_{Y,\epsilon}(y\mid a,x)\right|_{\epsilon=0},
\qquad
\int s(y,a,x)\,q_Y(y\mid a,x)\,dy\ =\ 0\ \ \forall (a,x).
\]
Then $\epsilon \mapsto P_{\epsilon}^{(T)}$ is DQM at $\epsilon = 0$ with score
\[
S_T(\bar O_T)\ =\ \left.\frac{d}{d\epsilon}\log \frac{dP^{(T)}_\epsilon}{dP^{(T)}}(\bar O_T)\right|_{\epsilon=0}
\ =\ \sum_{t=1}^T s(Y_t,A_t,X_t).
\]
Let $m_\epsilon(a,x):=\E_{P^{(T)}_\epsilon}[Y_t\mid A_t=a,X_t=x]$. A standard conditional-mean differentiation identity gives
\begin{align}
    \dot{m}(a,x) := \left.\frac{d}{d\epsilon}m_\epsilon(a,x)\right|_{\epsilon=0} &= \E_{P^{(T)}}\!\big[Y_t\,s(Y_t,A_t,X_t)\mid A_t=a,X_t=x\big]\nonumber\\
    &= \E_{P^{(T)}}\!\big[\varepsilon_t\,s(Y_t,A_t,X_t)\mid A_t=a,X_t=x\big].\label{eq:dotm_identity} 
\end{align}
Because $\{P_{\epsilon}^{(T)}: \epsilon \in (-\epsilon_0, \epsilon_0)\} \subset \mathcal{M}_T$, by Assumption~\ref{eq: characteristic_feature}, for each $\epsilon\in (-\epsilon_0, \epsilon_0)$, there exists a unique $\beta_{\epsilon}$ such that $m_{\epsilon}(a,x)\equiv \varphi(a,x)^{\top}\beta_{\epsilon}$. Then we claim that $\epsilon \mapsto \beta_{\epsilon}$ is differentiable at $\epsilon = 0$ with derivative $\dot{\beta}\in \mathbb{R}^{d_T}$. 
Indeed, by Assumption~\ref{eq: characteristic_feature}, we can choose $z_i$, $1\leq i\leq d_T$, such that $\{\varphi(z_i)\}_{1\leq i\leq d_T}$ form a linearly independent set. Let $\Phi$ denote the matrix whose $i$th row is $\varphi(z_i)^{\top}$. Then we have 
\begin{align*}
    \beta_{\epsilon} - \beta_T = \Phi^{-1}\begin{pmatrix}
        (m_{\epsilon} - m_0)(z_1)\\
        \vdots\\
        (m_{\epsilon} - m_0)(z_{d_T})
    \end{pmatrix}.
\end{align*}
The claim then follows from differentiability of $(m_{\epsilon} - m_0)(z_i)$ and continuity of $\Phi^{-1}$. In this case, 
\begin{equation}
\label{eq:dotm_linear}
\dot m(a,x)\ =\ \varphi(x,a)^\top \dot\beta
\qquad\forall(a,x).
\end{equation}

\paragraph{Step 2: Pathwise derivative of $\Psi_T$ along $\epsilon\mapsto q_{Y,\epsilon}$.}
Since $\Psi_T(P^{(T)})=\nu_T^{\top} \beta_T$, we have
\begin{equation}
\label{eq:dotpsi_equals_nu_dotbeta}
\left.\frac{d}{d\epsilon}\Psi_T(P^{(T)}_\epsilon)\right|_{\epsilon=0}
\ =\ \nu_T^{\top} \dot\beta.
\end{equation}
To express $\dot\beta$ in terms of the score $s$, multiply Eq.~\eqref{eq:dotm_linear} on the left by $\varphi(x,a)$ and integrate
with respect to the \emph{pooled} marginal $\bar h(a,x)\,dx$ (equivalently, take $T^{-1}\sum_{t=1}^T\E_{P^{(T)}}[\cdot]$):
\[
\frac{1}{T}\sum_{t=1}^T \E_{P^{(T)}}\!\left[\varphi(Z_t)\,\dot m(Z_t)\right]
\ =\
\frac{1}{T}\sum_{t=1}^T \E_{P^{(T)}}\!\left[\varphi(Z_t)\varphi(Z_t)^\top\right]\dot\beta
\ =\ \bar{\Sigma}_T\,\dot\beta.
\]Using Eq.~\eqref{eq:dotm_identity} and the law of iterated expectations,
\[
\frac{1}{T}\sum_{t=1}^T \E_{P^{(T)}}\!\left[\varphi(Z_t)\,\dot m(Z_t)\right]
\ =\
\frac{1}{T}\sum_{t=1}^T \E_{P^{(T)}}\!\left[\varphi(Z_t)\,\varepsilon_t\,s(Y_t,Z_t)\right].
\]
By Assumption~\ref{ass: source}, $\nu_T = \bar{\Sigma}_T\bar{\Sigma}_T^{\dag}\nu_T$. Thus from Eq.~\eqref{eq:dotpsi_equals_nu_dotbeta}, we have
\begin{align}
    \left.\frac{d}{d\epsilon}\Psi_T(P^{(T)}_\epsilon)\right|_{\epsilon=0}
    &= \nu_T^{\top}\bar{\Sigma}_T^{\dag}\frac{1}{T}\sum_{t=1}^{T}\mathbb{E}_{P^{(T)}}\left[\varphi(Z_t)\varphi(Z_t)^{\top}\right]\dot{\beta}\nonumber\\
&= \frac{1}{T}\sum_{t=1}^T
\E_{P^{(T)}}\!\left[\left(\nu_T^{\top} \bar{\Sigma}_T^{\dag}\varphi(Z_t)\,\varepsilon_t\right)\,s(Y_t,Z_t)\right]\nonumber\\
&= \frac{1}{T}\sum_{t=1}^{T}\mathbb{E}_{P^{(T)}}\left[\bar{\alpha}_T(Z_t)\varepsilon_t s(Y_t, Z_t)\right].\label{eq:dotpsi_innerproduct_single}
\end{align}

\paragraph{Step 3: Expressing Eq.~\eqref{eq:dotpsi_innerproduct_single} as a $L^2(P^{(T)})$-inner product. }
In $L_2(P^{(T)})$, the inner product of $\frac{1}{T}\sum_{t=1}^T \bar\alpha_T(Z_t)\varepsilon_t$ with the full score $S_T=\sum_{l=1}^T s(Y_l,Z_l)$ expands as

\begin{align*}
    &\E_{P^{(T)}}\!\left[\left(\frac{1}{T}\sum_{t=1}^T \bar\alpha_T(Z_t)\varepsilon_t\right)
\left(\sum_{l=1}^T s(Y_l,Z_l)\right)
\right]\ \\&=\ \frac{1}{T}\sum_{t=1}^T \E_{P^{(T)}}\!\big[\bar\alpha_T(Z_t)\varepsilon_t\,s(Y_t,Z_t)\big]\ +\ \frac{1}{T}\sum_{t\neq l}\E_{P^{(T)}}\!\big[\bar\alpha_T(Z_t)\varepsilon_t\,s(Y_l,Z_l)\big].
\end{align*}
For $t\neq l$, the cross-terms vanish by the law of iterated expectation and the fact that
$\E_{P^{(T)}}[\varepsilon_t\mid Z_t,\mathcal F_{t-1}]=\E_{P^{(T)}}[\varepsilon_t\mid Z_t]=0$:
if $l<t$, then $s(Y_l,Z_l)$ is $\mathcal F_{t-1}$-measurable and
\begin{align*}
    &\E_{P^{(T)}}\!\big[\bar\alpha_T(Z_t)\varepsilon_t\,s(Y_l,Z_l)\big]\\
    &= \E_{P^{(T)}}\!\left[s(Y_l,Z_l)\,\E_{P^{(T)}}[\bar\alpha_T(Z_t)\varepsilon_t\mid \mathcal F_{t-1}]\right]\\
    &= \E_{P^{(T)}}\!\left[s(Y_l,Z_l)\,\E_{P^{(T)}}[\bar\alpha_T(Z_t)\E_{P^{(T)}}[\varepsilon_t\mid Z_t]\mid \mathcal F_{t-1}]\right]
=0.
\end{align*}
The case $t<l$ is analogous (condition on $\mathcal F_{l-1}$).
Therefore
\begin{equation}
\label{eq:innerproduct_equals_diag}
\E_{P^{(T)}}\!\left[\left(\frac{1}{T}\sum_{t=1}^T\bar\alpha_T(Z_t)\varepsilon_t\right)S_T\right]
\ =\ 
\frac{1}{T}\sum_{t=1}^T \E_{P^{(T)}}\!\big[\bar\alpha_T(Z_t)\varepsilon_t\,s(Y_t,Z_t)\big].
\end{equation}

\paragraph{Step 4: We show $D^{\ast}_{T}$ is the canonical gradient in $\mathcal{M}_T^{\mathrm{res}}$. }
Putting Eq.~\eqref{eq:dotpsi_innerproduct_single} and Eq.~\eqref{eq:innerproduct_equals_diag} together, we obtain 
\begin{align*}
    \left.\frac{d}{d\epsilon}\Psi_T(P^{(T)}_\epsilon)\right|_{\epsilon=0} = \E_{P^{(T)}}\!\left[\left(\frac{1}{T}\sum_{t=1}^T\bar\alpha_T(Z_t)\varepsilon_t\right)S_T\right] = \E_{P^{(T)}}\!\left[D^{\ast}_{T}\,S_T\right],
\end{align*}
for $D^{\ast}_{T}$ given in Eq.~\eqref{eq:canonical_gradient}. Hence $D^{\ast}_{T}$ is a gradient for $\Psi_T$
along all smooth one-parameter parametric submodels $\mathcal{M}_T^{\mathrm{res}}$. 
Moreover, since 
\begin{align*}
    D^{\ast}_{T} = \frac{1}{T}\sum_{t=1}^{T}\bar{\alpha}_T(Z_t)(Y_t - h_T(Z_t)),
\end{align*}
where $\mathbb{E}[\bar{\alpha}_T(Z_t)(Y_t - h_T(Z_t)) \mid Z_t] = \bar{\alpha}_T(Z_t)\mathbb{E}[\varepsilon_t \mid Z_t] = 0$, we see that $D^{\ast}_{T,Y}$ lies in the tangent space of $\mathcal{M}_T^{\mathrm{res}}$ at $P^{(T)}$. Therefore it is the canonical gradient with respect to $\mathcal{M}_T^{\mathrm{res}}$. 

\paragraph{Step 5: Pathwise differentiability criterion.}
If $D^{\ast}_{T}\notin L_2(P^{(T)})$, then the linear functional
$s\mapsto \left.\frac{d}{d\epsilon}\Psi_T(P^{(T)}_\epsilon)\right|_{\epsilon=0}$
cannot be continuous on the tangent space of $\mathcal{M}_T^{\mathrm{res}}$ equipped with the $L_2(P^{(T)})$ norm,
so $\Psi_T$ is not pathwise differentiable at $P^{(T)}$.
Conversely, if $D^{\ast}_{T}\in L_2(P^{(T)})$, then \eqref{eq:dotpsi_innerproduct_single}--\eqref{eq:innerproduct_equals_diag}
show that the derivative is represented by the $L_2$ inner product with $D^{\ast}_{T}$, hence $\Psi_T$ is pathwise differentiable.
\end{proof}

\section{Asymptotic analysis of the one-step estimator}
\label{appendix:one_step_analysis}

\paragraph{Notation.} We omit subscript $T$ when it's clear from context, for instance, we sometimes write $\mathbb{E}_0$ for $\mathbb{E}_{P^{(T)}_0}$. For a positive definite, symmetric matrix $A\in \mathbb{R}^{d_T\times d_T}$ and $a\in \mathbb{R}^{d_T}$, we let $\|a\|_{A}:= \langle a, Aa\rangle^{\frac{1}{2}}_{\mathcal{H}}$. For a symmetric $A \in \mathbb{R}^{d_T\times d_T}$ and $\lambda\geq 0$, we also write $A_{\lambda} = A + \lambda I_{d_T}$. 

\begin{proof}[Proof of theorem \ref{thm:vme-and-an}]
We have
\begin{align*}
    R &= \psi_0(\hat{h}) - \psi_0(h_T) + \Tilde{P}^{(T)}\hat{\alpha}_T(Y - \hat{h})\\
    &\stackrel{(\ast)}{=} \Tilde{P}^{(T)} \widetilde \alpha_T (\hat{h}_{T,\lambda} - h_T) + \Tilde{P}^{(T)}\hat{\alpha}_T(h_T - \hat{h}_{T,\lambda})\\
    &= - \Tilde{P^{(T)}}(\hat{\alpha}_T - \widetilde \alpha_T)(\hat{h}_{T,\lambda} - h_T),
\end{align*}
where in $(\ast)$ we use $\Tilde{P}^{(T)}\varphi(z)\varepsilon = 0$. We have
\begin{align*}
    \hat{\Psi} - \psi(h_T) =& \left(P_{T} - \Tilde{P}^{(T)}\right)\hat{\alpha}_T(Y - \hat{h}_{T,\lambda}) - \Tilde{P}^{(T)}(\hat{\alpha}_T - \widetilde \alpha_T)(\hat{h}_{T,\lambda} - h_T)\\
    =& \left(P_T - \Tilde{P}^{(T)}\right)\widetilde \alpha_T(Y - h_T)\\
    &+ \left(P_{T} - \Tilde{P}^{(T)}\right)\left\{\hat{\alpha}_T(Y - \hat{h}_{T,\lambda}) - \widetilde \alpha_T (Y - h_T)\right\}\\
    &- \Tilde{P}^{(T)}(\hat{\alpha}_T - \widetilde \alpha_T)(\hat{h}_{T,\lambda} - h_T).
\end{align*}
Then, using the identity $\hat{a}\hat{b} - ab = (\hat{a} - a)b + a(\hat{b} - b) + (\hat{b} - b)(\hat{a} - a)$, we have
\begin{align}
    \hat{\Psi} - \psi(h_T) =& \left(P_T - \Tilde{P}^{(T)}\right)\widetilde \alpha_T(Y - h_T) \tag{I}\label{eq: tag_I}\\
    &+ \left(P_T - \Tilde{P}^{(T)}\right)(\hat{\alpha}_T - \widetilde \alpha_T)(Y - h_T) \tag{II} \label{eq: tag_II}\\
    &+ \left(P_T - \Tilde{P}^{(T)}\right)\widetilde \alpha_T (h_T - \hat{h}_{T,\lambda}) \tag{III} \label{eq: tag_III}\\
    &+ \left(P_T - \Tilde{P}^{(T)}\right)(\hat{\alpha}_T - \widetilde \alpha_T)(h_T - \hat{h}_{T,\lambda}) \tag{IV} \label{eq: tag_IV}\\
    &- \Tilde{P}^{(T)}(\hat{\alpha}_T - \widetilde \alpha_T)(\hat{h}_{T,\lambda} - h_T).\tag{V}  \label{eq: tag_V}
\end{align}
We have 
\begin{align*}
    \eqref{eq: tag_IV} + \eqref{eq: tag_V} &= - P_T(\hat{\alpha}_T - \widetilde \alpha_T)(\hat{h}_{T,\lambda} - h_T)\\
    \eqref{eq: tag_II} &= \left(P_T - \Tilde{P}^{(T)}\right)(\hat{\alpha}_T - \widetilde \alpha_T)\varepsilon. 
\end{align*}
Therefore the von Mises expansion can be explicitly written as
\begin{align}
    \hat{\Psi} - \Psi_T =& \left(P_T - \widetilde P^{(T)}\right) \widetilde \alpha_T (Y - h_T) \tag{A} \label{eq: tag_A}\\
    &+ \left(P_T - \widetilde P^{(T)}\right) \widetilde \alpha_T (h_T - \hat{h}_{T,\lambda}) \tag{B} \label{eq: tag_B}\\
    &+ \left(P_T - \widetilde P^{(T)}\right) (\widehat \alpha_T - \widetilde \alpha_T) \varepsilon \tag{C} \label{eq: tag_C}\\
    &- P_T (\widehat \alpha_T - \widetilde \alpha_T) (\widehat h_{T,\lambda} - h_T) \tag{D} \label{eq: tag_D}. 
\end{align}
\paragraph{Term \eqref{eq: tag_A}} \emph{Claim: }
\begin{align*}
    \frac{\sqrt{T}}{\widetilde \sigma_T}\eqref{eq: tag_A} \overset{d}{\rightarrow} \mathcal{N}(0,1).
\end{align*}
\emph{Proof of claim}. We verify the assumptions of Lemma~\ref{lem:mclt} for the natural filtration $\mathcal{F}_{T,t} = \mathcal{F}_t$. First, we verify \eqref{eq: QV}. We have
\begin{align*}
    \frac{T}{\widetilde \sigma_T^2}\frac{1}{T^2}\sum_{t=1}^{T}\widetilde \alpha_T(Z_t)^2 \sigma^2 &= \frac{\sigma^2}{T\widetilde \sigma_T^2} \sum_{t=1}^{T}\nu_T^{\top}\Tilde{\Sigma}_T^{-1}\varphi_T(Z_t)\varphi_T(Z_t)^{\top}\Tilde{\Sigma}_T^{-1}\nu_T\\
    &= \frac{\sigma^2}{\Tilde{\sigma}_T^2}\nu_T^{\top}\Tilde{\Sigma}_T^{-1}\hat{\Sigma}_T\Tilde{\Sigma}_T^{-1}\nu_T\\
    &= 1 + \frac{\sigma^2}{\Tilde{\sigma}_T^2}\nu_T^{\top}\Tilde{\Sigma}_T^{-1}(\hat{\Sigma}_T - \Tilde{\Sigma}_T)\Tilde{\Sigma}_T^{-1}\nu_T \overset{p}{\rightarrow} 1,
\end{align*}
where the last step follows by Assumption~\ref{ass: directional_stab_vme} and the fact that $\Tilde{\sigma}_T^2 = \sigma^2 \nu_T^{\top}\Tilde{\Sigma}_T^{-1}\nu_T$. We next verify \eqref{eq: Lindeberg} is verified by Assumption~\ref{ass: lindeberg_vme}. Hence the claim follows from Lemma~\ref{lem:mclt}. 

\paragraph{Term \eqref{eq: tag_B}} We have that
\begin{align*}
    \eqref{eq: tag_B} &= \nu_T^{\top} \widetilde \Sigma_T^{-1} \left(\widehat \Sigma_T - \widetilde \Sigma_T\right) \left(\widehat \beta_{T,\lambda} - \beta_T\right) \\
    &= \nu_T^{\top} \left(\widetilde \Sigma_T^{-1} - \widehat \Sigma_{T,1/T}^{-1}\right) \widehat \Sigma_T \left(\widehat \beta_{T,\lambda} - \beta_T\right) - \frac{1}{T} \nu_T^{\top} \widehat \Sigma_{T,1/T}^{-1} \left(\widehat \beta_{T,\lambda} - \beta_T\right) \\
    &= \nu_T^{\top} \left(\widetilde \Sigma_T^{-1} - \widehat \Sigma_{T,1/T}^{-1}\right)\hat{\Sigma}_T^{\frac{1}{2}}\hat{\Sigma}_T^{\frac{1}{2}}\left(\hat{\beta}_{T,\lambda} - \beta_T\right) - \frac{1}{T} \nu_T^{\top} \widehat \Sigma_{T,1/T}^{-1} \left(\widehat \beta_{T,\lambda} - \beta_T\right)\\
    &\leq \left\|\left(\widetilde\Sigma_T^{-1} - \widehat \Sigma_{T,1/T}^{-1}\right)\nu_T\right\|_{\hat{\Sigma}_T}\left\|\hat{\beta}_{T,\lambda} - \beta_T\right\|_{\hat{\Sigma}_T} - \frac{1}{T} \nu_T^{\top} \widehat \Sigma_{T,1/T}^{-1} \left(\widehat \beta_{T,\lambda} - \beta_T\right)\\
    &= \left\|\widehat \alpha_T - \widetilde \alpha_T\right\|_{L_2(P_T)} \left\|\hat{h}_{T,\lambda} - h_T\right\|_{L_2(P_T)} - \frac{1}{T} \nu_T^{\top} \widehat \Sigma_{T,1/T}^{-1} \left(\widehat \beta_{T,\lambda} - \beta_T\right),
\end{align*}
where the first line follows by definition of $P_T$ and $\widetilde P^{(T)}$, the second line follows from
\begin{align*}
    \widetilde \Sigma_T^{-1}\left(\hat{\Sigma}_T - \widetilde \Sigma_T\right) = \left(\widetilde \Sigma_T^{-1} - \hat{\Sigma}_{T,1/T}^{-1}\right)\hat{\Sigma}_T + \hat{\Sigma}_{T,1/T}^{-1}\hat{\Sigma}_T - I_{d_T} = \left(\widetilde \Sigma_T^{-1} - \hat{\Sigma}_{T,1/T}^{-1}\right)\hat{\Sigma}_T - \frac{1}{T}\hat{\Sigma}_{T,\frac{1}{T}}^{-1},
\end{align*}
we apply Cauchy Schwarz inequality in the second last line,   and in the last line we use
\begin{align*}
    \left\|\hat{h}_{T,\lambda} - h_T\right\|_{L_2(P_T)}^2 = \frac{1}{T}\sum_{t=1}^{T} \left(\hat{\beta}_{T,\lambda} - \beta_T\right)^{\top} \varphi(Z_t)\varphi(Z_t)^{\top}\left(\hat{\beta}_{T,\lambda} - \beta_T\right) = \left\|\hat{\beta}_{T,\lambda} - \beta_T\right\|_{\hat{\Sigma}_T}^2,
\end{align*}
and similarly $\left\|\widehat \alpha_T - \widetilde \alpha_T\right\|_{L_2(P_T)} = \left\|\left(\widetilde \Sigma_T^{-1} - \widehat \Sigma_{T,1/T}^{-1}\right)\nu_T\right\|_{\hat{\Sigma}_T}$.

\paragraph{Term \eqref{eq: tag_C}}
Recall that
\[
\eqref{eq: tag_C}
= (P_T-\widetilde P^{(T)})(\widehat\alpha_T-\widetilde\alpha_T)\,\varepsilon
= P_T\big[(\widehat\alpha_T-\widetilde\alpha_T)\,\varepsilon\big]
= \frac1T\sum_{t=1}^T \big(\widehat\alpha_T(Z_t)-\widetilde\alpha_T(Z_t)\big)\,\varepsilon_t,
\]
where the second equality uses $\widetilde P^{(T)}(a_2 \varphi(z)\varepsilon)=0$ by definition of $\widetilde P^{(T)}$ for any $z\in \mathcal{Z}$ and $a_2\in \mathbb{R}$. Let
\[
w_t := \widehat\alpha_T(Z_t)-\widetilde\alpha_T(Z_t),\qquad t=1,\dots,T.
\]
Then $(w_1,\dots,w_T)$ is measurable with respect to $Z_{1:T}$. Under Assumption~\ref{ass:gauss_noise},
conditional on $Z_{1:T}$ we have
\[
\sum_{t=1}^T w_t\varepsilon_t \ \Big|\ Z_{1:T}\ \sim\ \mathcal{N}\!\left(0,\ \sigma^2\sum_{t=1}^T w_t^2\right).
\]
Therefore,
\[
\eqref{eq: tag_C}\ \Big|\ Z_{1:T}
\ \sim\ \mathcal{N}\!\left(0,\ \frac{\sigma^2}{T^2}\sum_{t=1}^T w_t^2\right)
= \mathcal{N}\!\left(0,\ \frac{\sigma^2}{T}\cdot \frac1T\sum_{t=1}^T w_t^2\right).
\]
Moreover,
\[
\frac1T\sum_{t=1}^T w_t^2
=
\|\widehat\alpha_T-\widetilde\alpha_T\|_{L_2(P_T)}^2,
\]
since $P_T$ is the empirical measure on $(Z_t)_{t\le T}$.
Hence, for any $\delta\in(0,1)$, using the standard Gaussian tail bound
$\mathbb{P}(|N(0,1)|\ge x)\le 2e^{-x^2/2}$, we obtain
\[
\mathbb{P}\!\left(
\left|\eqref{eq: tag_C}\right|
\le
\sigma\,\|\widehat\alpha_T-\widetilde\alpha_T\|_{L_2(P_T)}
\sqrt{\frac{2\log(2/\delta)}{T}}
\ \Bigm|\ Z_{1:T}
\right)\ \ge\ 1-\delta.
\]
Since the right-hand side is measurable in $Z_{1:T}$, the same inequality holds unconditionally:
\[
\mathbb{P}\!\left(
\left|\eqref{eq: tag_C}\right|
\le
\sigma\,\|\widehat\alpha_T-\widetilde\alpha_T\|_{L_2(P_T)}
\sqrt{\frac{2\log(2/\delta)}{T}}
\right)\ \ge\ 1-\delta.
\]
Consequently,
\[
\eqref{eq: tag_C}
=
O_{P^{(T)}}\!\left(\frac{\sigma}{\sqrt T}\,\|\widehat\alpha_T-\widetilde\alpha_T\|_{L_2(P_T)}\right)
=
O_{P^{(T)}}\!\left(\frac{1}{\sqrt T}\,\|\widehat\alpha_T-\widetilde\alpha_T\|_{L_2(P_T)}\right),
\]
where the second equality uses that $\sigma$ is a fixed constant.

\paragraph{Term \eqref{eq: tag_D}} From Cauchy-Schwarz, $\eqref{eq: tag_D} \leq \left\|\hat{\alpha}_T - \widetilde \alpha_T\right\|_{L_2(P_T)} \left\|\hat{h}_{T,\lambda} - h_T\right\|_{L_2(P_T)}$.
\end{proof}

\section{Upper Bound on Stability Rate under LinUCB Sampling}
\label{appendix:stability_rate}

\begin{proof}[Proof of proposition \ref{prop:linucb_stab}]
We omit polylogs throughout the proof. We omit the dependence of $d$ on $T$ notationwise and write $d = d_T$. Let $\widehat \Lambda_T := T \widehat \Sigma_T$, let $v = v_1$ be the top eigenvector of $\widehat \Lambda_T$, let $v_i$, $i=2,\ldots,d$ be its eigenvectors corresponding to the non-leading eigenvalues in descending order, denote $\lambda_i$, $i=1,\ldots,d$ its eigenvalues ordered in descending order, i.e.
\begin{align*}
    \widehat\Lambda_T = \sum_{i=1}^{d}\lambda_i v_iv_i^{\top}.
\end{align*}
Define $Q_\star := vv^\top$ and $Q_{\perp} := I_d - Q_\star$. Let $e_1 := \beta_T / \| \beta_T \|_2$. We use $\gamma$ for the $\beta$ in \citet{fan2025statisticalinferenceadaptivesampling}. Define $P_{\ast} = e_1e_1^{\top}$ and $P_{\perp} = I_d - P_{\ast}$, as the projection onto the true signal direction (and its orthogonal complement), and let 
    \begin{align}
        \widetilde \Lambda_T &:= \omega_1 P_\star + \bar \omega P_{\perp},\\
        \check \Lambda_T &:= \lambda_1 Q_\star + \bar \lambda Q_\perp,
    \end{align}
with $\bar \lambda := \bar \omega := \gamma \sqrt{T / d}$, $\omega_1 := T$. We assume $\sqrt{d / T} = o(1)$. We'll establish Proposition~\ref{prop:linucb_stab} with $\widetilde \Sigma_T = \frac{1}{T}\widetilde \Lambda_T$. 

\paragraph{Preliminary facts.} 
\begin{itemize}
    \item From \citet[Eq. (27), Eq. (28)]{fan2025statisticalinferenceadaptivesampling}, under assumptions \ref{ass: bounded_features}-\ref{ass: sub_gaussian_noise}-\ref{ass: large_exploration},  with probability $\geq 1 - \frac{1}{\log T}$, 
\begin{align}
    &\|v -  e_1 \| \lesssim \sqrt{\frac{d}{\bar \lambda}} = \frac{d^{3/4}}{\gamma^{1/2} T^{1/4}}.\\
    &\lambda_{i} = (1+\Delta_{T,i})\sqrt{\frac{2\gamma^2 T}{d + 1}},
\end{align}
where we have
\begin{align*}
    |\Delta_{T,i}| \lesssim \varepsilon_{\mathrm{bulk}} = o_T(1).
\end{align*}
by assumption.
    \item From $\|Z_t\| = 1$ for every $t$, 
    \begin{align*}
        \mathrm{Tr}(\check \Lambda_T) = T + d = \lambda_1 + (d - 1)\bar \lambda.
    \end{align*}
    Therefore, under $d=o(T)$, the eigenvalue gap $\delta := \omega_1 - \lambda_1$ satisfies
    \begin{align*}
        \delta = (d-1)\gamma\sqrt{T/d} - d = O(\gamma \sqrt{T d}). 
    \end{align*}
    \item It holds that
    \begin{align*}
        \| Q_\perp P_\star \| &= \| e_1 e_1^\top - Q_\star e_1 e_1^\top\| \\
        &= \| (e_1 - Q_\star e_1) e_1^\top \| \\
        &\leq \| e_1 - Q_\star v \| \\
        &\leq \| e_1 - v\|.
    \end{align*}
    Similarly $\| Q_\star P_\perp  \| \leq \| e_1 - v\|$.
    \item For any $\sigma_i$, $i=1,\ldots,d$, we have
    \begin{align*}
        \left\lVert \sum_i \sigma_i v_i v_i^\top e_1 e_1^\top \right\rVert_{\mathrm{op}} \leq \left(\sum_{j\geq 2}\sigma_j^2 (v_j^{\top}e_1)^2\right)^{\frac{1}{2}} \leq \left(\max_{j\geq 2}|\sigma_j|\right)\left(\sum_{j\geq 2}(v_j^{\top}e_1)^2\right)^{\frac{1}{2}}
    \end{align*}
    We have
    \begin{align*}
        \sum_{j\geq 2}(v_j^{\top}e_1)^2 &= \left\|e_1 - e_1^\top v_1 v_1\right\|^2 = \left\|e_1 - \mathrm{Proj}(e_1\mid \mathrm{Span}(v_1))\right\|^2 \leq \|e_1 - v\|^2. 
    \end{align*}
    Thus we have
    \begin{align}
        \left\lVert \sum_i \sigma_i v_i v_i^\top e_1 e_1^\top \right\rVert_{\mathrm{op}} \leq \max_{j\geq 2}|\sigma_j|\|v - e_1\|. \label{eq:rotation_bound}
    \end{align}
\end{itemize}

We consider the following decomposition of the target quantity. We have
\begin{align}
    &\|\widehat \Lambda_T^{1/2} \left( \widehat \Lambda_T^{-1} - \widetilde \Lambda_T^{-1} \right) a \| \nonumber\\
    \leq& \|\check \Lambda_T^{1/2} \left( \check \Lambda_T^{-1} - \widetilde \Lambda_T^{-1} \right) a \| \label{eq: stability_rate_term_1}\\ 
    &+ \|(\widehat \Lambda_T^{-1/2} \widetilde\Lambda_T^{1/2} - \widehat \Lambda_T^{1/2} \widetilde\Lambda_T^{-1/2}) \widetilde\Lambda_T^{-1/2} a - (\check \Lambda_T^{-1/2} \widetilde\Lambda_T^{1/2} - \check \Lambda_T^{1/2} \widetilde\Lambda_T^{-1/2}) \widetilde\Lambda_T^{-1/2} a \|. \label{eq: stability_rate_term_2}
\end{align}

We bound the two terms above in steps 1 and 2 below, respectively.

\paragraph{Step 1.} We have that
\begin{align*}
    & \check \Lambda_T^{1/2} \left( \check \Lambda_T^{-1} - \widetilde\Lambda_T^{-1} \right) a \\
    =& (\mathrm{I}) + (\mathrm{II}) + (\mathrm{III}) + (\mathrm{IV}),
\end{align*}
with 
\begin{align*}
    (\mathrm{I}) &:= \left(\lambda_1^{-1/2} \omega_1^{1/2} - \lambda_1^{1/2} \omega_1^{-1/2} \right) \omega_1^{-1/2} Q_\star P_\star a,\\
    (\mathrm{II}) &:= \left( \bar \lambda^{-1/2}  \bar \omega^{1/2} - \bar \lambda^{1/2} \bar \omega^{-1/2} \right) \bar \omega^{-1/2} Q_\perp P_\perp a,\\
    (\mathrm{III}) &:= \left( \bar \lambda^{-1/2} \omega_1^{1/2} - \bar \lambda^{1/2} \omega_1^{-1/2} \right) \omega_1^{-1/2} Q_{\perp} P_\star a, \\
    (\mathrm{IV}) &:= \left(\lambda_1^{-1/2} - \lambda_1^{1/2} \bar \omega^{-1} \right) Q_\star P_\perp a.
\end{align*}
\begin{itemize}
    \item First term. Using $\sqrt{d / T} = o(1)$, we have that
\begin{align*}
    \|(\mathrm{I})\| &\lesssim \left\lvert \sqrt{\frac{1}{1 - \delta / T}} - \sqrt{1 - \delta / T} \right\rvert \frac{\|P_\star a \|}{\sqrt{T}}   \\
    & \lesssim \frac{\delta}{T} \frac{\|P_\star a \|}{\sqrt{T}}  \\ 
    & \lesssim \gamma \sqrt{d} \frac{1}{T} \| P_\star a \|.
\end{align*}
\item Second term: since $\bar \omega = \bar \lambda$, $\mathrm{(II)} = 0$.
\item Third term. Using $\sqrt{d / T} = o(1)$, we have that
\begin{align*}
    \|(\mathrm{III})\| &\lesssim \left( \left( \gamma \sqrt{\frac{T}{d }} \right)^{-1/2} \sqrt{T} + \left( \gamma \sqrt{\frac{T}{d }} \right)^{1/2} \frac{1}{\sqrt{T}} \right) \frac{1}{\sqrt{T}} \frac{d^{3/4}}{\sqrt{\gamma} T^{1/4}} \|P_\star a\| \\
    &= \left(\frac{d}{\gamma \sqrt{T}} + \frac{\sqrt{d}}{T} \right) \|P_\star a\|.
\end{align*} 
\item Fourth term:
\begin{align*}
    \|(\mathrm{IV})\| &\lesssim \left(\frac{1}{\sqrt{T - \delta}} + \frac{\sqrt{T - \delta}}{\gamma \sqrt{T}} \sqrt{d} \right) \frac{d^{3/4}}{\sqrt{\gamma} T^{1/4}} \| P_\perp a \| \\ 
    &\lesssim \left( \frac{1}{\sqrt{T}} + \frac{\sqrt{d}}{\gamma} \right) \frac{d^{3/4}}{\sqrt{\gamma} T^{1/4}} \| P_\perp a \|
\end{align*}
\end{itemize}
Collecting the bounds above and using $\gamma^{-1} =o(1)$, we then have
\begin{align*}
    \| \left( \check \Lambda_T^{-1} - \widetilde \Lambda_T^{-1} \right) a \|_{\check \Lambda_T} \lesssim \left( \frac{\gamma \sqrt{d}}{T} + \frac{d}{\gamma \sqrt{T}}\right) \|P_\star a \|+ \left( \frac{1}{\sqrt{T}} + \frac{\sqrt{d}}{\gamma} \right) \frac{d^{3/4}}{\sqrt{\gamma} T^{1/4}} \| P_\perp a \|.
\end{align*}

\paragraph{Step 2.} We have that 
\begin{align*}
 &\left(\widehat \Lambda_T^{-1/2} \widetilde \Lambda_T^{1/2} - \widehat \Lambda_T^{1/2} \widetilde\Lambda_T^{-1/2}\right) \widetilde\Lambda_T^{-1/2} a - \left(\check \Lambda_T^{-1/2} \widetilde\Lambda_T^{1/2} - \check \Lambda_T^{1/2} \widetilde\Lambda_T^{-1/2}\right) \widetilde\Lambda_T^{-1/2} a \\
=& \left(\hat{\Lambda}_T^{-\frac{1}{2}} - \check{\Lambda}_T^{-\frac{1}{2}}\right)a  - \left(\hat{\Lambda}_T^{\frac{1}{2}}\widetilde\Lambda_T^{-1} - \check{\Lambda}_T^{\frac{1}{2}}\widetilde\Lambda_T^{-1}\right)a\\
= & (\mathrm{V}) + (\mathrm{VI}) 
\end{align*}
with
\begin{align*}
    (\mathrm{V}) &:=  \sum_{i=2}^d (\lambda_i^{-1/2} - \bar \lambda^{-1/2}) v_i v_i^\top (P_\star + P_\perp) a  \\
    (\mathrm{VI}) &:=  \sum_{i=2}^d (\lambda_i^{1/2} - \bar \lambda^{1/2}) v_i v_i^\top \left(  \frac{1}{T} P_\star + \frac{1}{\bar \lambda} P_\perp  \right) a. 
\end{align*}
From Eq.~\eqref{eq:rotation_bound},
\begin{align*}
\|(\mathrm{V})\| &\leq \max_{j}\left|\lambda_j^{-\frac{1}{2}} - \bar{\lambda}^{-\frac{1}{2}}\right|\|v-e_1\|\|P_{\ast}a\| + \left\|\sum_{i=2}^{d}\left(\lambda_i^{-\frac{1}{2}} - \bar{\lambda}^{-\frac{1}{2}}\right)v_iv_i^{\top}\right\| \|P_{\perp}a\|\\
&\lesssim \frac{\bar \lambda}{\bar \lambda^{3/2}} \varepsilon_{\mathrm{bulk}} \| v-e_1 \| \|P_\star a\| + \frac{\bar \lambda}{\bar \lambda^{3/2}} \varepsilon_{\mathrm{bulk}} \| P_\perp a\| \\
&\lesssim \varepsilon_{\mathrm{bulk}} \left( \frac{d}{\gamma \sqrt{T}} \| P_\star a \| + \frac{d^{1/4}}{\sqrt{\gamma} T^{1/4}}  \| P_\perp a\| \right) 
\end{align*}
and
\begin{align*}
\|(\mathrm{VI})\| & \lesssim \frac{\bar \lambda}{\bar \lambda^{1/2}} \varepsilon_{\mathrm{bulk}} \frac{1}{T} \| v - e_1 \| \|P_\star a\| + \frac{\bar \lambda}{\bar \lambda^{1/2} \bar \lambda} \varepsilon_{\mathrm{bulk}} \|P_\perp a\| \\
& \lesssim \varepsilon_{\mathrm{bulk}} \left( \frac{\sqrt{d}}{T} \| P_\star a \| + \frac{d^{1/4}}{\sqrt{\gamma} T^{1/4}}  \| P_\perp a\| \right) 
\end{align*}

\paragraph{Collecting the bounds.} We have that
\begin{align*}
&\left\lVert ( \widehat\Sigma_T^{-1} - \widetilde\Sigma_T^{-1}) a   \right\rVert_{\widehat\Sigma_T}\\
&= \sqrt{T} \left\lVert (\widehat \Lambda_T^{-1} - \widetilde\Lambda_T^{-1}) a \right\rVert_{\widehat \Lambda_T}
\\
&= \left( \gamma \sqrt{\frac{d}{T}} + \frac{d}{\gamma} + \sqrt{\frac{d}{T}} \varepsilon_{\mathrm{bulk}} \right) \| P_\star a\| + 
\left( \sqrt{\frac{d}{T}} + \frac{d}{\gamma} + \varepsilon_{\mathrm{bulk}} \right) \frac{(T d)^{1/4}}{\sqrt{\gamma}}\|P_\perp a\|.
\end{align*}
Now observing that 
\begin{align*}
    \widetilde\sigma_T^2 = \|P_\star a\|^2 + \frac{\sqrt{Td}}{
    \gamma} \| P_\perp a\|^2
\end{align*}
yields the desired result. 
\end{proof}

\section{Efficiency Theory}
\label{app:efficiency}
In this section, we prove Theorem~\ref{thm:lan} and Theorem~\ref{thm: conv_thm}. We adapt the efficiency theory along a sequence of experiments introduced in \citet[Section I]{vanderlaan2026nonparametricinstrumentalvariableinference}. Both our setting and their setting consider a sequence of statistical models with diverging Fisher information, which is indexed by horizon $T$ in our case and the number of instruments $K$ in their case. However, the key distinction is that they consider a factorized model \citet{bickel1993efficient} due to the independence between different units, whereas we consider a longitudinal model where $g_t(a_t|x_t, \bar{o}_{t-1})$ factors induce intertemporal dependencies. This necessitates certain adaptations in our construction.

\subsection{Local asymptotic normality along least favorable submodels}
\label{appendix:least_favorable_submodel}
Recall $z = (a,x)$ and we defined
\begin{align*}
    \bar{P}_T f(z,y) &= \int \bar{h}_T(z)q(y|z)f(z, y)\;\mathrm{d}z\;\mathrm{d}y,\\
    \bar{h}_T(z) &= \frac{1}{T}\sum_{t=1}^{T}\bar{g}_t(a|x)q_X(x).\\
    D_T^{\ast} &= \frac{1}{T}\sum_{t=1}^{T}\bar{\alpha}_T(Z_t)(Y_t - h_T(Z_t)) \in L_0^2(P^{(T)}),
\end{align*}
and we employ the empirical process notation $P^{(T)}[f]$ for $\mathbb{E}_{P^{(T)}}[f]$, $f : \bar{O}_T \mapsto \mathbb{R}$, where we only integrate over the randomness of the arguments of $f$, and not $f$ itself. Then we have
\begin{align*}
    P^{(T)}[(D_T^{\ast})^2] = P^{(T)}\left(\frac{1}{T}\sum_{t=1}^{T}\bar{\alpha}_T(A_t, X_t)\varepsilon_t\right)^2 \overset{(\ast)}{=} \frac{1}{T^2}\sum_{t=1}^{T}\mathbb{E}_{P^{(T)}}\left[(\bar{\alpha}_T(A_t, X_t)\varepsilon_t)^2\right] 
    \overset{(\ast\ast)}{=} \frac{1}{T}\bar{\sigma}_T^2. 
\end{align*}
where in $(\ast)$ we use conditional mean zero of $\varepsilon_t$ and law of total expectation, and in $(\ast\ast)$, $\bar{\sigma}_T = \sigma\sqrt{\nu_T^{\top}\bar{\Sigma}_T^{\dag}\nu_T}$ is defined in Section~\ref{sec: asymptotic_analysis}. To see why $(\ast\ast)$ holds, we note that 
\begin{align*}
    \bar{\Sigma}_T = \mathbb{E}_{P^{(T)}}\left[\frac{1}{T}\sum_{t=1}^{T}\varphi(Z_t)\varphi(Z_t)^{\top}\right] &= \int \varphi(z)\varphi(z)^{\top}\left(\frac{1}{T}\sum_{t=1}^{T}\bar{g}_t(a|x)q_X(x)\right)\;\mathrm{d}a\;\mathrm{d}x
\end{align*}
hence 
\begin{align}
    \frac{1}{T}\sum_{t=1}^{T}\mathbb{E}_{P^{(T)}}\left[(\bar{\alpha}_T(Z_t)\varepsilon_t)^2\right] &= \frac{1}{T}\sum_{t=1}^{T}\mathbb{E}_{P^{(T)}}\left[\varepsilon_t^2 \nu_T^{\top}\bar{\Sigma}_T^{\dag}\varphi(Z_t)\varphi(Z_t)^{\top}\bar{\Sigma}_T^{\dag}\nu_T\right] \nonumber\\
    &= \frac{1}{T}\sum_{t=1}^{T}\mathbb{E}_{P^{(T)}}\left[\mathbb{E}[\varepsilon_t^2\mid \mathcal{F}_{t-1}, Z_t] \nu_T^{\top}\bar{\Sigma}_T^{\dag}\varphi(Z_t)\varphi(Z_t)^{\top}\bar{\Sigma}_T^{\dag}\nu_T\right] \nonumber\\
    &= \sigma^2 \nu_T^{\top}\bar{\Sigma}_T^{\dag}\bar{\Sigma}_T \bar{\Sigma}_T^{\dag}\nu_T \nonumber\\
    &= \sigma^2 \nu_T^{\top}\bar{\Sigma}_T^{\dag}\nu_T, \label{eq: sigma_bar_T_nu_expression}
\end{align}
where we use the homoschedastic noise assumption $\mathbb{E}[\varepsilon_t^2 \mid \mathcal{F}_{t-1}, Z_t] = \sigma^2$ (Assumption~\ref{ass: homoschedastic}). We now construct an explicit least favorable submodel satisfying Definition~\ref{def:lfs}. The construction follows the classical Le Cam--H{\'a}jek quadratic mean differentiability path \citep{lecam1986asymptotic}, adapted to the present setting by taking the score to be the canonical gradient $D_T^*$. This provides a concrete example of a submodel along which the local asymptotic normality expansion of Theorem~\ref{thm:lan} holds. We adapt the factorizable model construction in \citet[Lemma 17]{vanderlaan2026nonparametricinstrumentalvariableinference}, noting that while $g_t$ factors induce dependence between time points, it suffices to consider perturbations of the repeated $q_Y(y_t|z_t)$ factors in order to obtain a one-parameter submodel with $D_T^{\ast} = D_{T,Y}^{\ast}$ score at the origin. In Subsection~\ref{subsection: LAN}, we'll leverage the factorizability to express the log-likelihood ratio process in terms of a martingale, and then obtain local asymptotic normality via a careful application of martingale limit theory \citep{hall1980}. 

\begin{lem}[Le Cam--H{\'a}jek QMD path]\label{lem:lch_qmd}
Fix horizon $T\geq 1$ and assume Assumption~\ref{ass: homoschedastic}. We define a one-parameter family $\eta \mapsto q_{Y, \eta}$, for $|\eta|\leq \delta$, with score $s_{T}(z, y) = \frac{1}{T}\bar{\alpha}_T(z)(y - h_T(z))$, by 
\begin{align}
    q_{Y, \eta}(y|z):= \frac{\big(1 + \frac{\eta}{2T}\bar{\alpha}_T(z)(y - h_T(z))\big)^2}{1 + \frac{\eta^2}{4}\frac{\bar{\sigma}_T^2}{T^2}}q_{Y}(y|z).\label{eq: q_Y_eta_def}
\end{align}We then define a one-dimensional parametric submodel $\{P^{(T)}_\eta:|\eta|\le\delta\}$ via
\begin{align*}
    \frac{\mathrm{d} P^{(T)}_{\eta}}{\mathrm{d}\mu^{(T)}}(\bar{o}_T) = \prod_{t=1}^{T}q_X(x_t)g_t(a_t|x_t, \bar{o}_{t-1})q_{Y, \eta}(y_t|a_t,x_t).
\end{align*}
Recall $\mu$ is the base measure on $\mathcal{Z}\times \mathcal{Y}$. Then i) $q_{Y,\eta}(y\mid z)\bar{h}_T(z)$ is a valid probability density with respect to $\mu$, ii), $\{q_{Y,\eta}\bar{h}_T: |\eta|\leq \delta\}$ is QMD at $\eta = 0$ with score $\frac{1}{T}\bar{\alpha}_T(z)(y - h_T(z))$, and iii), the remainder term
\begin{align*}
    R_{t,\eta}(y,z):= \sqrt{q_{Y,\eta}(y\mid z)\bar{h}_T(z)} - \sqrt{q_Y(y\mid z)\bar{h}_T(z)} - \frac{\eta}{2T}\bar{\alpha}_T(z)(y - h_T(z))\sqrt{q_{Y}(y\mid z)\bar{h}_T(z)}
\end{align*}
satisfies $\left\|R_{t,\eta}\right\|_{L^2(\mu)}^2 \leq \left(1 - \sqrt{1 + \frac{\eta^2}{4}\frac{\bar{\sigma}_T^2}{T^2}}\right)^2 \leq \eta^4 \frac{\bar{\sigma}_T^4}{T^4}$. 
\end{lem}
\begin{proof}
For this proof $t$ is fixed. We let 
    \begin{align*}
        \xi(y,z)^2 = q_Y(y|z)\bar{h}_T(z).
    \end{align*}
    This has the property that for any $f : \mathcal{Y}\times \mathcal{Z}\to \mathbb{R}$, we have
    \begin{align}
        \int \int \xi(y,z)^2\mathrm{d}\mu(y,z)f(y,z) = \bar{P}^{(T)}[f] = \mathbb{E}_{P^{(T)}}\left[\frac{1}{T}\sum_{t=1}^{T}f(Y_t, Z_t)\right]\label{eq: xi_property}. 
    \end{align}
    We define, for $\eta\in \mathbb{R}$
    \begin{align*}
        \xi_\eta (y,z) := \xi(y,z)\left(1 + \frac{\eta}{2T}\bar{\alpha}_T(z)(y - h_T(z))\right). 
    \end{align*}
    We readily verify from Eq.~\eqref{eq: xi_property} that 
    \begin{align*}
        &\int \xi(y,z)^2\;\mathrm{d}\mu(y,z) = 1\\
        &\int \xi(y,z)^2\frac{1}{T}\bar{\alpha}_T(z)(y - h_T(z))\;\mathrm{d}\mu(y,z) = 0\\
        &\int \xi(y,z)^2\left(\frac{1}{T}\bar{\alpha}_T(z)(y - h_T(z))\right)^2\;\mathrm{d}\mu(y,z) = \mathbb{E}_{P^{(T)}}\left[\frac{1}{T}\sum_{t=1}^{T}\left(\frac{1}{T}\bar{\alpha}_T(Z_t)(Y_t - h_T(Z_t))\right)^2\right] = \frac{\bar{\sigma}_T^2}{T^2}\\
        &\int \xi_{\eta}(y,z)^2\;\mathrm{d}\mu(y,z) = 1 + \frac{\eta^2}{4}\frac{\bar{\sigma}_T^2}{T^2}.
    \end{align*}
    Therefore Eq.~\eqref{eq: q_Y_eta_def} defines a valid conditional probability. Thus we've shown i). We have
    \begin{align*}
        \log q_{Y,\eta}(y|z) = 2\log \left(1 + \frac{\eta}{2T}\bar{\alpha}_T(z)(y - h_T(z))\right)- \log \left(1 + \frac{\eta^2}{4}\frac{\bar{\sigma}_T^2}{T^2}\right) + \log q_Y(y|z). 
    \end{align*}
    We now compute its derivative at $\eta = 0$:
    \begin{align*}
        \left.\frac{\partial}{\partial\eta}\right|_{\eta = 0}\log q_{Y,\eta}(y|z) &= \frac{1}{T}\bar{\alpha}_T(z)(y - h_T(z)) = s_T(z,y).
    \end{align*}
    Thus we've shown ii). We now control the remainder $\left\|R_{t,\eta}\right\|_{L^2(\mu)}^2$ to establish quadratic mean differentiability at $\eta = 0$. We obtain via direct algebraic manipulations that 
    \begin{align*}
        R_{t,\eta}(y,z) = \left(\frac{1}{\sqrt{1 + \frac{\eta^2}{4}\frac{\bar{\sigma}_T^2}{T^2}}} - 1\right)\xi(y,z)\left(1 + \frac{\eta}{2}s_T(z,y)\right). 
    \end{align*}
    Therefore 
    \begin{align*}
        \|R_{t,\eta}\|^2_{L^2(\mu)} &= \left(\frac{1}{\sqrt{1 + \frac{\eta^2}{4}\frac{\bar{\sigma}_T^2}{T^2}}} - 1\right)^2\mathbb{E}_{P^{(T)}}\left[\left(1 + \frac{\eta}{2}s_T(Z_t,Y_t)\right)^2\right]\\
        &= \left(\frac{1}{\sqrt{1 + \frac{\eta^2}{4}\frac{\bar{\sigma}_T^2}{T^2}}} - 1\right)^2\left(1 + \frac{\eta^2}{4}\frac{\bar{\sigma}_T^2}{T^2}\right)\\
        &= \left(1 - \sqrt{1 + \frac{\eta^2}{4}\frac{\bar{\sigma}_T^2}{T^2}}\right)^2 \leq \eta^4 \frac{\bar{\sigma}_T^4}{T^4}. 
    \end{align*}
    where we use Eq.~\eqref{eq: xi_property} in the first line, and the last inequality follows from $|1-\sqrt{1+x}|\leq x$. 
\end{proof}

This construction supplies an example of a QMD one-dimensional submodel with
\begin{align*}
    \partial_{\eta}\mid_{\eta = 0}\log P^{(T)}_{\eta}(\bar{o}_T) = D^{\ast}_T(\bar{o}_T).
\end{align*}
The QMD of $P_{\eta}^{(T)}$ is implied by the QMD of $q_{Y,\eta}$. 

\subsection{Local asymptotic normality along least favorable submodels}
\label{subsection: LAN}
In this section, we adapt \citet[Theorem 7.2]{vanderVaart1998asymptotic} and \citet[Theorem 17]{vanderlaan2026nonparametricinstrumentalvariableinference} in order to prove Theorem~\ref{thm:lan}, as presented below.
\begin{proof}
Fix $\epsilon >0$. We define $I_T = \frac{\bar{\sigma}_T^2}{T}$ and $\eta = \epsilon I_T^{-\frac{1}{2}}$. We define the random variable
    \begin{align*}
        W_t := 2\left[\sqrt{\frac{q_{Y, \eta}(Y_t\mid Z_t)}{q_{Y}(Y_t\mid Z_t)}} - 1\right]. 
    \end{align*}
    The log-likelihood ratio admits the expansion
    \begin{align*}
        \log \frac{p_{\eta}^{(T)}(\bar{O}_T)}{p^{(T)}(\bar{O}_T)} = 2\sum_{t=1}^{T}\log\left(1 + \frac{1}{2}W_t\right). 
    \end{align*}
    We aim to show that, as $T\to \infty$, we have
    \begin{align}
    \label{eq: wt_sum_prob}
        \sum_{t=1}^{T}W_t = \eta \frac{1}{T}\sum_{t=1}^{T}\bar{\alpha}_T(Z_t)(Y_t - h_T(Z_t)) - \frac{\epsilon^2}{4} + o_p(1). 
    \end{align}
    Then, we use the Taylor expansion $\log(1+x) = x - \frac{x^2}{2} + x^2 R(2x)$ where $R(x) \rightarrow 0$ as $x\to 0$, to obtain
    \begin{align*}
        \log \frac{p_{\eta}^{(T)}(\bar{O}_T)}{p^{(T)}(\bar{O}_T)} &= 2\sum_{t=1}^{T}\log \left(1 + \frac{1}{2}W_t\right)\\
        &= \sum_{t=1}^{T} W_t - \frac{1}{4}W_t^2 + \frac{1}{2}W_t^2R(W_t)\\
        &= \eta \frac{1}{T}\sum_{t=1}^{T}\bar{\alpha}_T(Z_t)(Y_t - h_T(Z_t)) - \frac{1}{4}\sum_{t=1}^{T}W_t^2 + \frac{1}{2}\sum_{t=1}^{T}W_t^2R(W_t) - \frac{\epsilon^2}{4} + o_p(1). 
    \end{align*}
    We will establish the following statements
    \begin{align}
        \label{eq: QV_wtsq}&\sum_{t=1}^{T}W_t^2 \overset{p}{\rightarrow} \epsilon^2, \tag{QV}\\
        \label{eq: mclt_lan_nts}&\frac{\eta}{T}\sum_{t=1}^{T}\bar{\alpha}_T(Z_t)(Y_t - h_T(Z_t)) \overset{d}{\rightarrow} \mathcal{N}(0, \epsilon^2), \tag{MCLT}\\
        \label{eq: wt_remainder_op1}&\sum_{t=1}^{T}W_t^2R(W_t) = o_p(1), \tag{REM}
    \end{align}
    where Eq.~\eqref{eq: QV_wtsq} is a quadratic variation term, Eq.~\eqref{eq: mclt_lan_nts} will follow from a Martingale Central Limit Theorem with stable quadratic variation \citep[Theorem 3.2]{hall1980}, and Eq.~\eqref{eq: wt_remainder_op1} will follow from a negligibility condition. Putting the above statements together, we obtain that
    \begin{align*}
        \log \frac{p_{\eta}^{(T)}(\bar{O}_T)}{p^{(T)}(\bar{O}_T)} \overset{p}{\rightarrow}\mathcal{N}\left(-\frac{\epsilon^2}{2}, \epsilon^2\right).
    \end{align*}

    \paragraph{Proof of Eq.~\eqref{eq: wt_sum_prob}.} It suffices to show that the mean and variance of 
    \begin{align*}
        \sum_{t=1}^{T}W_t - \frac{\eta}{T}\sum_{t=1}^{T}\bar{\alpha}_T(Z_t)(Y_t - h_T(Z_t)) + \frac{\epsilon^2}{4}
    \end{align*}
    converges to zero. By the martingale structure, we can write
    \begin{align*}
        &\mathrm{Var}_{P^{(T)}}\left(\sum_{t=1}^{T}W_t - \frac{\eta}{T}\sum_{t=1}^{T}\bar{\alpha}_T(Z_t)(Y_t - h_T(Z_t))\right)\\
        &= \sum_{t=1}^{T}\mathrm{Var}_{P^{(T)}}\left[W_t - \frac{\eta}{T}\bar{\alpha}_T(Z_t)(Y_t - h_T(Z_t))\right]\\
                &\leq T\mathbb{E}_{P^{(T)}}\left[\frac{1}{T}\sum_{t=1}^{T}\left(W_t - \frac{\eta}{T}\bar{\alpha}_T(Z_t)(Y_t - h_T(Z_t))\right)^2\right]\\
        &=T \int \left(2\sqrt{\frac{q_{Y,\eta}(y\mid z)}{q_Y(y\mid z)}} -2 - \frac{\eta}{T}\bar{\alpha}_T(z)(y - h_T(z)) \right)^2q_Y(y\mid z)\bar{h}_T(z)\mathrm{d}\mu(a,x,y)\\
        &= 4T\int \left(\sqrt{q_{Y,\eta}(y\mid z)} - \sqrt{q_Y(y\mid z)}\left(1 +  \frac{\eta}{2T}\bar{\alpha}_T(z)(y - h_T(z))\right) \right)^2\bar{h}_T(z)\mathrm{d}\mu(a,x,y)\\
        &\stackrel{(\ast)}{\leq} 4\eta^4 T\frac{\bar{\sigma}_T^4}{T^4} = 4\frac{\epsilon^4}{T} = o(1). 
        \end{align*}
        where $(\ast)$ follows from the remainder bound in Lemma~\ref{lem:lch_qmd}. We now show that mean vanishes. Using the fact that the score has mean zero, we have 
        \begin{align*}
            &\mathbb{E}_{P^{(T)}}\left[\sum_{t=1}^{T}W_t\right]\\
            &= T\mathbb{E}_{P^{(T)}}\left[\frac{1}{T}\sum_{t=1}^{T}2\sqrt{\frac{q_{Y, \eta}(Y_t\mid Z_t)}{q_Y(Y_t\mid Z_t)}} - 2\right]\\
            &= T\int \left(2\sqrt{\frac{q_{Y,\eta}(y\mid z)}{q_Y(y\mid z)}}-2\right)\bar{h}_T(z)q_Y(y\mid z)\mathrm{d}\mu(y,z)\\
            &= 2T \int \sqrt{q_{Y,\eta}(y\mid z)q_Y(y\mid z)}\bar{h}_T(z)\mathrm{d}\mu(y,z)  - 2T\\
            &= - T \int \left(\sqrt{q_{Y,\eta}(y\mid z)\bar{h}_T(z)} - \sqrt{q_Y(y\mid z)\bar{h}_T(z)}\right)^2\mathrm{d}\mu(y,z). 
    \end{align*}
    If we define
    \begin{align*}
        A &= \sqrt{q_{Y,\eta}(y\mid z)\bar{h}_T(z)} - \sqrt{q_Y(y\mid z)\bar{h}_T(z)}\\
        B &= R_{t,\eta} - A = -\frac{\eta}{2T}\bar{\alpha}_T(z)(y - h_T(z))\sqrt{q_Y(y\mid z)\bar{h}_T(z)},
    \end{align*}
    We now use $\|R_{t,\eta}\|_{L^2(\mu)}\leq \eta^2\frac{\bar{\sigma}_T^2}{T^2}$ and $\eta = \epsilon I_T^{-\frac{1}{2}} = \epsilon \bar{\sigma}_T^{-1}T^{\frac{1}{2}}$ to obtain
    \begin{align*}
        \|R_{t,\eta}\|_{L^2(\mu)} \leq \epsilon^2 \bar{\sigma}_T^{-2}T\frac{\bar{\sigma}_T^2}{T^2} = \frac{\epsilon^2}{T}.
    \end{align*}
    We compute
    \begin{align*}
        \|B\|^2_{L^2(\mu)} &= \frac{\eta^2}{4T^2}\int (\bar{\alpha}_T(z)(y - h_T(z)))^2 q_Y(y\mid z)\bar{h}_T(z)\mathrm{d}\mu(y,z)\\
        &= \frac{\eta^2}{4T^2}\mathbb{E}\left[\frac{1}{T}\sum_{t=1}^{T}(\bar{\alpha}_T(Z_t)(Y_t - h_T(Z_t)))^2\right]\\
        &= \frac{\eta^2}{4T^2}\bar{\sigma}_T^2 = \frac{\epsilon^2\bar{\sigma}_T^{-2}T\bar{\sigma}_T^2}{4T^2} = \frac{\epsilon^2}{4T}. 
    \end{align*}
    where in the last step we substitute in the definition of $\eta = \epsilon I_T^{-\frac{1}{2}}$. Then we have
    \begin{align*}
        \left|-T \|A\|^2_{L^2(\mu)} + T\|B\|^2_{L^2(\mu)}\right| &= \left|-T \|R_{t,\eta} - B\|^2_{L^2(\mu)} + T \|B\|^2_{L^2(\mu)}\right|\\
        &= \left|-T\|R_{t,\eta}\|^2_{L^2(\mu)} + 2T\langle B, R_{t,\eta}\rangle_{L^2(\mu)}\right|\\
        &\leq T\|R_{t, \eta}\|^2_{L^2(\mu)} + 2T\|B\|_{L^2(\mu)}\|R_{t,\eta}\|_{L^2(\mu)}.
    \end{align*}
    Hence we have
    \begin{align*}
        \left|-T \|A\|^2_{L^2(\mu)} + T\|B\|^2_{L^2(\mu)}\right| \leq \frac{\epsilon^4}{T} + \frac{\epsilon}{2\sqrt{T}}2T\frac{\epsilon^2}{T} = \frac{\epsilon^3}{\sqrt{T}} = o(1). 
    \end{align*}
    Hence, we have
    \begin{align*}
        \left|\mathbb{E}_{P^{(T)}}\left[\sum_{t=1}^{T}W_t\right] + \frac{\epsilon^2}{4}\right| = \left|-T\|A\|^2_{L^2(\mu)} + T\|B\|^2_{L^2(\mu)}\right| = o(1),
    \end{align*}
    as desired.

    \paragraph{Proof of Eq.~\eqref{eq: mclt_lan_nts}.} We verify the assumptions of Lemma~\ref{lem:mclt}. Firstly we verify \eqref{eq: QV}. We have
    \begin{align*}
        \sum_{t=1}^{T}\frac{\eta^2}{T^2}\mathbb{E}\left[\bar{\alpha}_T(Z_t)^2(Y_t - h_T(Z_t))^2\mid \mathcal{F}_{T,t-1}\right] &= \frac{\epsilon^2\sigma^2}{\bar{\sigma}_T^2}\frac{1}{T}\sum_{t=1}^{T}\bar{\alpha}_T(Z_t)^2\\
        &= \frac{\epsilon^2\sigma^2}{\bar{\sigma}_T^2}\nu_T^{\top}\bar{\Sigma}_T^{\dag}\left(\frac{1}{T}\sum_{t=1}^{T}\varphi(Z_t)\varphi(Z_t)^{\top}\right)\bar{\Sigma}_T^{\dag}\nu_T\\
        &= \frac{\epsilon^2\sigma^2}{\bar{\sigma}_T^2}\nu_T^{\top}\bar{\Sigma}_T^{\dag}\hat{\Sigma}_T\bar{\Sigma}_T^{\dag}\nu_T,
    \end{align*}
    By Assumption~\ref{ass: lan_stability}, we have
    \begin{align}
    \label{eq: qv_epsilon_sq}
        \sum_{t=1}^{T}\frac{\eta^2}{T^2}\bar{\alpha}_T(Z_t)^2(Y_t - h_T(Z_t))^2 \overset{p}{\rightarrow}\epsilon^2.
    \end{align}
    Secondly, we require
    \begin{align*}
        \sum_{t=1}^{T}\mathbb{E}\left[\eta^2 s_T(Z_t, Y_t)^2 1[|\eta s_T(Z_t, Y_t)| > \epsilon]\mid \mathcal{F}_{T,t-1}\right] \overset{p}{\rightarrow} 0.
    \end{align*}
    Moreover, \eqref{eq: Lindeberg} is verified by Assumption~\ref{ass: conditional_lind_eff}. 
    
    \paragraph{Proof of Eq.~\eqref{eq: QV_wtsq}.} We have
    \begin{align*}
        \sum_{t=1}^{T}W_t^2 &= \sum_{t=1}^{T}4\left[\sqrt{\frac{q_{Y,\eta}(Y_t\mid Z_t)}{q_Y(Y_t\mid Z_t)}} - 1\right]^2\\
        &= \sum_{t=1}^{T}4\left[\frac{R_{t,\eta}(Y_t, Z_t)}{\sqrt{q_Y(Y_t\mid Z_t)\bar{h}_T(Z_t)}} + \frac{\eta}{2T}\bar{\alpha}_T(Z_t)(Y_t - h_T(Z_t))\right]^2\\
        &= \sum_{t=1}^{T}4\frac{R_{t,\eta}(Y_t, Z_t)^2}{q_Y(Y_t\mid Z_t)\bar{h}_T(Z_t)}\\
        &+ \sum_{t=1}^{T}4\frac{\eta^2}{4T^2}\bar{\alpha}_T(Z_t)^2(Y_t - h_T(Z_t))^2\\
        &+ \sum_{t=1}^{T}\frac{4R_{t,\eta}(Y_t, Z_t)}{\sqrt{q_Y(Y_t\mid Z_t)\bar{h}_T(Z_t)}}\frac{\eta}{T}\bar{\alpha}_T(Z_t)(Y_t - h_T(Z_t)). 
    \end{align*}
    We have
    \begin{align*}
        \mathbb{E}_{P^{(T)}}\left[\sum_{t=1}^{T}\frac{R_{t,\eta}(Y_t, Z_t)^2}{q_Y(Y_t\mid Z_t)\bar{h}_T(Z_t)}\right] &= \sum_{t=1}^{T}\int\int \frac{R_{t,\eta}(y,z)^2}{q_Y(y\mid z)\bar{h}_T(z)}q_Y(y\mid z)q_X(x)g_t(a\mid x)\mathrm{d}z\mathrm{d}y\\
        &= T\|R_{t,\eta}\|^2_{L^2(\mu)} \leq \eta^4\frac{\bar{\sigma}_T^4}{T^3} = \frac{\epsilon^4}{T} = o(1). 
    \end{align*}
    Since convergence in expectation implies convergence in probability, we have
    \begin{align*}
        \sum_{t=1}^{T}4\frac{R_{t,\eta}(Y_t, Z_t)^2}{q_Y(Y_t\mid Z_t)\bar{h}_T(Z_t)} \overset{p}{\rightarrow} 0,
    \end{align*}
    We now conclude via Eq.~\eqref{eq: qv_epsilon_sq} and a Cauchy-Schwarz inequality that $\sum_{t=1}^{T}W_t^2 \overset{p}{\rightarrow} \epsilon^2$. 
    
    \paragraph{Proof of Eq.~\eqref{eq: wt_remainder_op1}. } This is implied by $\max_{t}|W_t| = o_T(1)$ as shown in the proof of \citet[Theorem 7.2]{vanderVaart1998asymptotic}. In our setting, this holds due to $I_T = \frac{\bar{\sigma}_T^2}{T} \to \infty$ as $T\to \infty$, hence $\eta \to 0$ as $T\to \infty$.  
\end{proof}

\subsection{Proof of Convolution theorem}
We prove Theorem~\ref{thm: conv_thm} following a standard route
(QMD, LAN, Third Le Cam lemma, matching and equivariance,
Anderson lemma), with the adaptive localization
$\eta_T=\epsilon\sqrt{T}/\bar\sigma_T$.
Throughout, we implicitly assume that the information index
$I_T:=\bar\sigma_T^2/T\to\infty$, so that $\eta_T\to0$ for each fixed
$\epsilon$.

\begin{proof}[Proof of Theorem~\ref{thm: conv_thm}]
Fix $\epsilon\in\mathbb R$ and define the log-likelihood ratio
\[
\Lambda_T(\epsilon)
:=
\log\frac{dP^{(T)}_{\epsilon\sqrt{T}/\bar\sigma_T}}{dP^{(T)}}(\bar O_T).
\]
By Theorem~\ref{thm:lan},
\begin{equation}\label{eq:lan_app}
\Lambda_T(\epsilon)
=
\epsilon\,\Delta_T-\frac{\epsilon^2}{2}+o_{P^{(T)}}(1),
\qquad
\Delta_T:=\frac{\sqrt{T}}{\bar\sigma_T}D_T^*(\bar O_T)
\Rightarrow \mathcal N(0,1)
\ \text{under }P^{(T)}.
\end{equation}

In particular, the LAN expansion \eqref{eq:lan_app} implies that for each
fixed $\epsilon\in\mathbb R$, the sequence of local laws
$\{P^{(T)}_{\epsilon\sqrt{T}/\bar\sigma_T}\}_{T\ge1}$ is contiguous with
respect to $\{P^{(T)}\}_{T\ge1}$. Define the normalized estimation error under the baseline law $P^{(T)}$,
\begin{equation}\label{eq:phiT_def}
\phi_T(\bar O_T)
:=
\frac{\sqrt{T}}{\bar\sigma_T}\Big(\hat\Psi_T-\Psi_T(P^{(T)})\Big).
\end{equation}
Regularity (Definition~\ref{def:regularity}) implies tightness of
$\{\phi_T\}$ under each local law $P^{(T)}_{\epsilon\sqrt{T}/\bar\sigma_T}$,
hence in particular under $P^{(T)}$. Along a subsequence (not relabeled),
\begin{equation}\label{eq:subseq_joint}
(\phi_T,\Delta_T)\Rightarrow(S,\Delta)
\quad\text{under }P^{(T)},
\qquad \Delta\sim\mathcal N(0,1).
\end{equation}

\paragraph{Step 1: Third lemma and the family of limit laws.}
From \eqref{eq:lan_app} and \eqref{eq:subseq_joint},
\[
(\phi_T,\Lambda_T(\epsilon))\Rightarrow (S,\ \epsilon\Delta-\epsilon^2/2)
\quad\text{under }P^{(T)}.
\]
Le Cam’s third lemma implies that, under $P^{(T)}_{\epsilon\sqrt{T}/\bar\sigma_T}$,
\begin{equation}\label{eq:limit_family}
\phi_T \Rightarrow L_\epsilon,
\qquad
L_\epsilon(B)
=
\mathbb E\!\left[
\mathbf 1_{\{S\in B\}}
\exp\!\left(\epsilon\Delta-\tfrac{\epsilon^2}{2}\right)
\right],
\quad B\in\mathcal B(\mathbb R).
\end{equation}

\paragraph{Step 2: Matching in the Gaussian shift experiment.}
By the matching theorem \citep[Theorem~7.10]{vanderVaart1998asymptotic},
there exist a measurable map $g:\mathbb R\times[0,1]\to\mathbb R$ and an
auxiliary random variable $U\sim\mathrm{Unif}(0,1)$ independent of
everything else such that, for $X_\epsilon\sim\mathcal N(\epsilon,1)$,
\begin{equation}\label{eq:matching_app}
g(X_\epsilon,U)\sim L_\epsilon,
\qquad \forall \epsilon\in\mathbb R.
\end{equation}

\paragraph{Step 3: Regularity $\Rightarrow$ equivariance-in-law.}
Define the (normalized) local centering term
\begin{equation}\label{eq:BT_def}
B_T(\epsilon)
:=
\frac{\sqrt{T}}{\bar\sigma_T}\Big(
\Psi_T(P^{(T)}_{\epsilon\sqrt{T}/\bar\sigma_T})-\Psi_T(P^{(T)})
\Big).
\end{equation}
By pathwise differentiability of $\Psi_T$ at $P^{(T)}$ (Assumption~\ref{ass: pathwise_differentiability}),
for any one-dimensional submodel $\{P^{(T)}_\eta\}$ through $P^{(T)}$ with score
$s_T=\partial_\eta\log p^{(T)}_\eta|_{\eta=0}$,
\begin{equation}\label{eq:pathwise_expansion}
\Psi_T(P^{(T)}_\eta)-\Psi_T(P^{(T)})
=
\eta\,P^{(T)}\!\big[D_T^*\,s_T\big]+o(\eta),
\qquad \eta\to0.
\end{equation}
For a least favorable submodel in the sense of Definition~\ref{def:lfs}, we have
$s_T=D_T^*$. Substituting into \eqref{eq:pathwise_expansion} gives
\begin{equation}\label{eq:pathwise_lfs}
\Psi_T(P^{(T)}_\eta)-\Psi_T(P^{(T)})
=
\eta\,P^{(T)}\!\big[(D_T^*)^2\big]+o(\eta).
\end{equation}
We now take the local parameter
\[
\eta_T:=\epsilon\sqrt{T}/\bar\sigma_T=\epsilon/\sqrt{I_T}.
\]
Since $I_T\to\infty$, we have $\eta_T\to0$, and the pathwise differentiability
expansion \eqref{eq:pathwise_lfs} applies along this sequence. Plugging
$\eta=\eta_T$ into \eqref{eq:pathwise_lfs} and then into \eqref{eq:BT_def}
yields

\begin{equation}\label{eq:BT_precise}
B_T(\epsilon)
=
\epsilon\,\frac{T}{\bar\sigma_T^2}\,P^{(T)}\!\big[(D_T^*)^2\big]+o(1).
\end{equation}
Moreover, since $\eta_T\to0$ and $\sqrt{T}/\bar\sigma_T=1/\sqrt{I_T}\to0$,
the remainder term satisfies
\[
\frac{\sqrt{T}}{\bar\sigma_T}\,o(\eta_T)=o(1).
\]

By construction of $\bar\sigma_T, P^{(T)}[(D_T^*)^2]=\bar\sigma_T^2/T$, we have that
\begin{equation}\label{eq:BT_slope}
B_T(\epsilon)=\epsilon+o(1)
\qquad\text{for each fixed }\epsilon.
\end{equation}

Regularity (Definition~\ref{def:regularity}) asserts that, for every fixed
$\epsilon\in\mathbb R$,
\begin{equation}\label{eq:reg_app}
\frac{\sqrt{T}}{\bar\sigma_T}
\Big(\hat\Psi_T-\Psi_T(P^{(T)}_{\epsilon\sqrt{T}/\bar\sigma_T})\Big)
\Rightarrow L
\quad\text{under } P^{(T)}_{\epsilon\sqrt{T}/\bar\sigma_T},
\end{equation}
for a law $L$ not depending on $\epsilon$.
Since $\phi_T - B_T(\epsilon)$ is exactly the left-hand side of \eqref{eq:reg_app},
we obtain
\[
\phi_T - B_T(\epsilon)\Rightarrow L
\quad\text{under } P^{(T)}_{\epsilon\sqrt{T}/\bar\sigma_T}.
\]
Combining this with \eqref{eq:limit_family} and \eqref{eq:BT_slope} yields
\[
L_\epsilon(\cdot-\epsilon)=L_0(\cdot),
\qquad \forall \epsilon\in\mathbb R.
\]
Using \eqref{eq:matching_app}, this is equivalent to
\begin{equation}\label{eq:equivariance_app}
g(X_\epsilon,U)-\epsilon \ \stackrel d=\ g(X_0,U),
\qquad \forall \epsilon\in\mathbb R,
\end{equation}
i.e.\ $g$ is equivariant-in-law in the Gaussian shift experiment.

\paragraph{Step 4: Anderson decomposition and convolution.}
By \citep[Proposition~8.4]{vanderVaart1998asymptotic}, equivariance-in-law implies
there exists a random variable $\xi$, independent of
$X_0\sim\mathcal N(0,1)$, such that
\[
g(X_0,U)\stackrel d= X_0+\xi.
\]
Let $M:=\mathcal L(\xi)$. Then
\[
L_0=\mathcal L(g(X_0,U))=\mathcal N(0,1)\ast M.
\]
Since regularity identifies the (subsequential) limit law under $P^{(T)}$ as
$L=L_0$, we obtain the convolution representation
\[
L=\mathcal N(0,1)\ast M.
\]
\end{proof}
\section{Auxiliary results}
We use the following version of the martingale central limit theorem from \citet[Corollary 3.1]{hall1980}. 
\begin{lem}
\label{lem:mclt}
    Suppose $\{(X_{n,i}, \mathcal{F}_{n,i}) : i\in [k_n], n\in \mathbb{N}\}$ is a martingale difference array with the nested property $\mathcal{F}_{n,i}\subset \mathcal{F}_{n+1, i}$ for all $i\in[k_n]$ and $n\in \mathbb{N}$. We further assume that
    \begin{align}
        & (\forall \epsilon >0)\; \sum_{i\in[k_n]}\mathbb{E}[X_{n,i}^21[|X_{n,i}|> \epsilon]\mid \mathcal{F}_{n, i-1}]\overset{p}{\rightarrow} 0 \tag{Lindeberg} \label{eq: Lindeberg}\\
        & \sum_{i\in [k_n]}\mathbb{E}\left[X_{n,i}^2\mid \mathcal{F}_{n,i-1}\right] \overset{p}{\rightarrow} 1. \tag{QV}\label{eq: QV}
    \end{align}
    Then $\sum_{i\in [k_n]}X_{n,i}\overset{d}{\rightarrow}\mathcal{N}(0,1)$. 
\end{lem}